\documentclass{article} 
\usepackage{iclr2027_conference,times}

\usepackage{amsmath,amsfonts,bm}

\def\eqref#1{equation~\ref{#1}}

\def\1{\bm{1}}

\DeclareMathAlphabet{\mathsfit}{\encodingdefault}{\sfdefault}{m}{sl}
\SetMathAlphabet{\mathsfit}{bold}{\encodingdefault}{\sfdefault}{bx}{n}

\usepackage{hyperref}
\usepackage{url}

 \usepackage{graphicx}
 \usepackage{wrapfig}
\usepackage{graphicx}
\usepackage{subcaption}
\usepackage{multirow}
\usepackage{wrapfig}
\usepackage{booktabs}
\usepackage[utf8]{inputenc} 
\usepackage[T1]{fontenc}    
\usepackage{hyperref}

\usepackage{url}            
\usepackage{booktabs}       
\usepackage{amsfonts}       
\usepackage{nicefrac}       
\usepackage{microtype}      
\usepackage{xcolor}         
\usepackage{amsthm}
\usepackage{amsmath}
\newtheorem{theorem}{Theorem}

\newtheorem{corollary}{Corollary}
\usepackage{cleveref}      
\usepackage{booktabs}
\usepackage{colortbl}
\usepackage{xcolor}

\title{Making Transformers Logical} 

\author{%
Mingze Ma\\
Australian Institute for Machine Learning\\
Adelaide University\\
\texttt{mingze.ma@adelaide.edu.au}
\And
Hemanth Saratchandran\\
Australian Institute for Machine Learning\\
Adelaide University\\
\phantom{\texttt{simon.lucey@adelaide.edu.au}}
\AND
Cameron Gordon\\
Australian Institute for Machine Learning\\
Adelaide University\\
\phantom{\texttt{mingze.ma@adelaide.edu.au}}
\And
\hspace{20pt}
Simon Lucey\\
\hspace{20pt}
Australian Institute for Machine Learning\\
\hspace{20pt}
Adelaide University\\
\hspace{20pt}
\texttt{simon.lucey@adelaide.edu.au}
}

\iclrfinalcopy 
\begin{document}

\maketitle

\begin{abstract}

Recent work has shown that transformers operating on logical tasks will often fit a simple function that interpolates between observed examples. This means that they often fail to generalize to unseen instances of logical boolean or arithmetic problems. Generalization on the Unseen (GOTU) has therefore become a stress test for understanding the inductive bias of transformers. In this paper, we demonstrate that poor logical extrapolation is not necessarily an inherent property of transformers, and that the desired behavior can be obtained by directly modifying this inductive bias in the network architecture. We show that including a simple attention mask is often sufficient to enable near-perfect extrapolation on boolean tasks and improves sample efficiency on arithmetic tasks. We show that although a similar bias may be encoded by initializing query and key parameters, a mask-based encoding is preserved during training. We additionally explore how a similar mechanism may be applied in vision. These results show the potential for using mask-based priors to encode logical inductive biases in transformers. 
\end{abstract}

\section{Introduction}
\label{sec:intro}
Attention provides transformers \citep{vaswani2017attention} with a powerful mechanism for learning long-range token interactions directly from data. This flexibility has been central to their success in natural language processing \citep{vaswani2017attention, radford2019language, guo2025deepseek} and computer vision \citep{dosovitskiy2020image, liu2021swin, peebles2023scalable}. However, it can also make their inductive bias difficult to control. Recent work on generalization on the unseen (GOTU) \citep{abbe2024generalization} illustrates this issue. The GOTU problem is a test of out-of-distribution generalization where data is divided into `seen' and `unseen' components. \citet{abbe2024generalization} show that where seen data points are fully learned by transformers there is a preference for a simple interpolative function (a minimum-degree interpolator) rather than equally plausible higher degree functions. \Cref{fig:seen_unseen_predictions-1} highlights this issue: training on three points for the boolean parity problem leads to a simple decision boundary and misclassifies the unseen data point. This raises the question: \textbf{Can we change the inductive bias of transformers to improve performance on unseen examples}? 




A natural way to approach this problem is to modify self-attention. Since self-attention controls how tokens interact, altering attention patterns provides a direct mechanism to encode structural priors for desired logical interactions. We therefore investigate whether enforcing a well-designed self-attention pattern can guide a Transformer to generalize across unseen portions of a logical Boolean function.

\textit{If} a logical rule can be encoded in self-attention this raises the question of \textit{how} we can encode this as a prior. While attention patterns can be shaped through masks \citep{duan2024attention,vaswani2017attention,roy2021efficient}, an alternative strategy would be to encode this interaction through a structured initialization \citep{trockman2023mimetic,zheng2025structured}. Structured initialization methods encode a desired attention pattern by modifying the query and key projections, using either prior knowledge about the dataset \citep{zheng2025structured, shinnick2026learnimagesproceduralwarmup} or posterior observations from well-trained attention maps \citep{trockman2023mimetic}. However, since query and key projections must learn task-dependent representations during training, the injected structure may be washed out based on seen examples in a dataset. As a result, it is unclear whether structured initialization methods, even when encoding a known structural prior, can learn a desired rule or if it will instead default to a simple interpolative solution in the GOTU or arithmetic problem setting.



\begin{figure}[h]
    \centering
    \includegraphics[width=\textwidth]{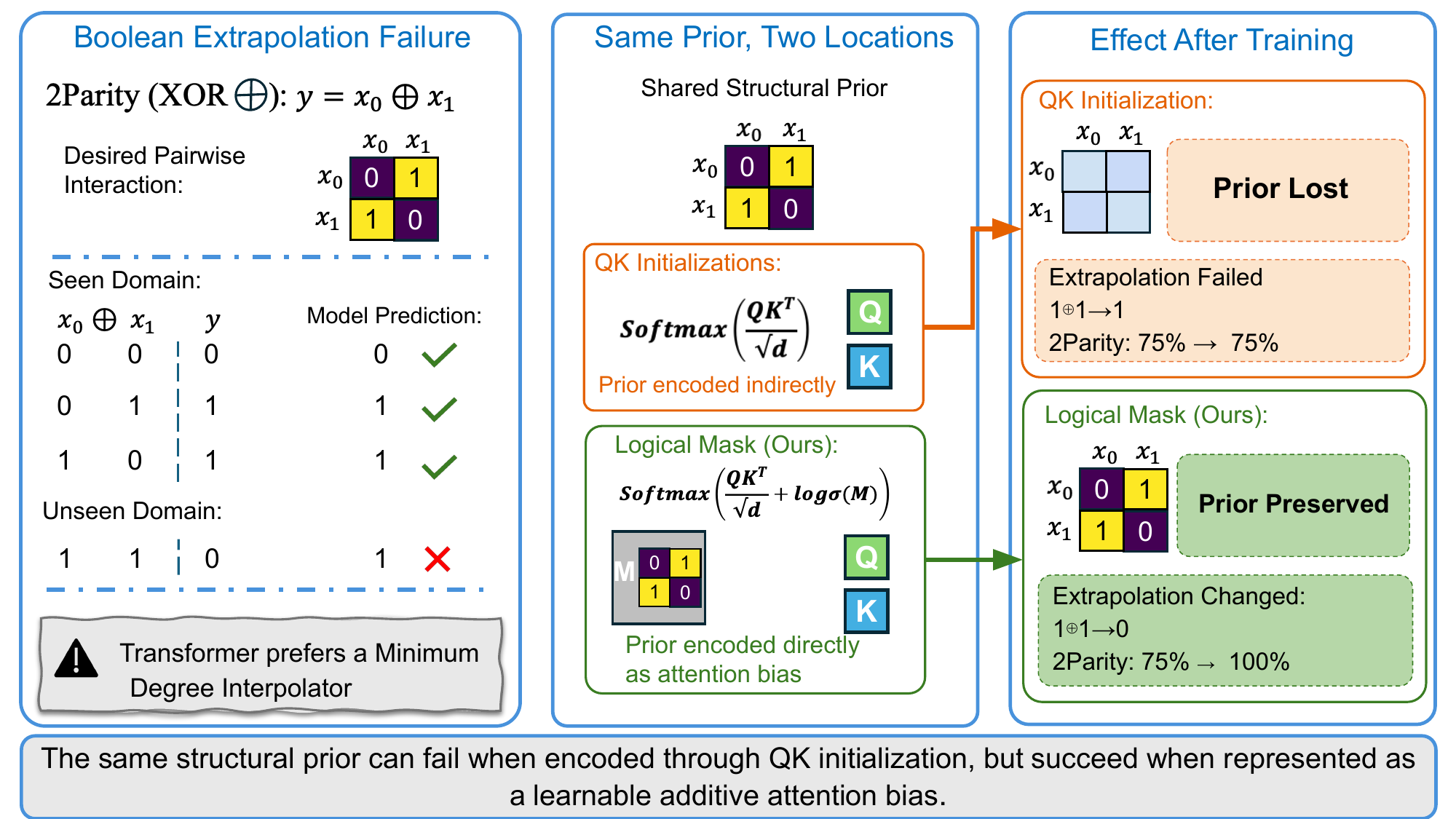}
\caption{
\textbf{Where the prior lives determines whether it changes extrapolation.}
Boolean GOTU tasks expose a failure mode in which transformers fit the observed domain but extrapolate according to the wrong rule, often a simpler minimum-degree interpolator. We compare two ways of injecting the same structural prior into attention. QK-based initialization encodes the prior indirectly through content-dependent query-key projections, choosing $Q$ and $K$ so that $\mathrm{Softmax}(QK^\top/\sqrt{d})$ approximates the desired pattern. Mask-based initialization instead places the prior directly in the attention logits as a finite, learnable additive bias. After training, QK-based priors can be washed out, whereas mask-based priors remain explicit and persistent, leading to substantially improved extrapolation.
}   
    \label{fig:teaser}
\end{figure}

In this paper we first establish that the use of a self-attention mask with a known logical interaction can improve the performance of transformers on logical tasks with unseen examples in the GOTU setting. In particular, we demonstrate that where the prior is placed matters. QK-based structured initialization methods can produce the desired interaction pattern at initialization but do not necessarily preserve this information during training and do not reliably alter the extrapolation of the vanilla transformer. Secondly, we demonstrate that this strategy improves the sample efficiency of transformers on arithmetic tasks such as multiplication and addition. Finally, we ask whether the use of a learnable attention mask can have benefits for traditional vision problems and show that it has advantages over alternative ways to encode a structural prior such as through mimetic or QK initialization.

In summary, our main contributions are:
\begin{itemize}
    \item We demonstrate that the tendency for transformers to learn minimum-degree solutions can be avoided by using appropriately designed attention masks. We show that this can enable higher-order interactions required to learn to solve logical problems in the GOTU setting. 
    

    \item We show that although a similar inductive bias can be encoded at initialization through query and key parameters, that this does not persist over training. In contrast, using an attention mask provides a harder constraint and preserves this inductive bias during training. 
    

    \item We provide a local gradient-level analysis for how inductive biases encoded in the attention mask are more persistent compared with initialization of query and key parameters.

\end{itemize}

We validate our approach for logical GOTU tasks including boolean and arithmetic problems. In addition, we demonstrate that encoding inductive biases through a masking approach has benefits over structured initialization approaches that extend to vision tasks. Our work highlights the potential for using mask-based priors for encoding specific logical inductive biases in transformers.

\begin{figure}[!ht]
\centering
\includegraphics[width=0.9\textwidth]{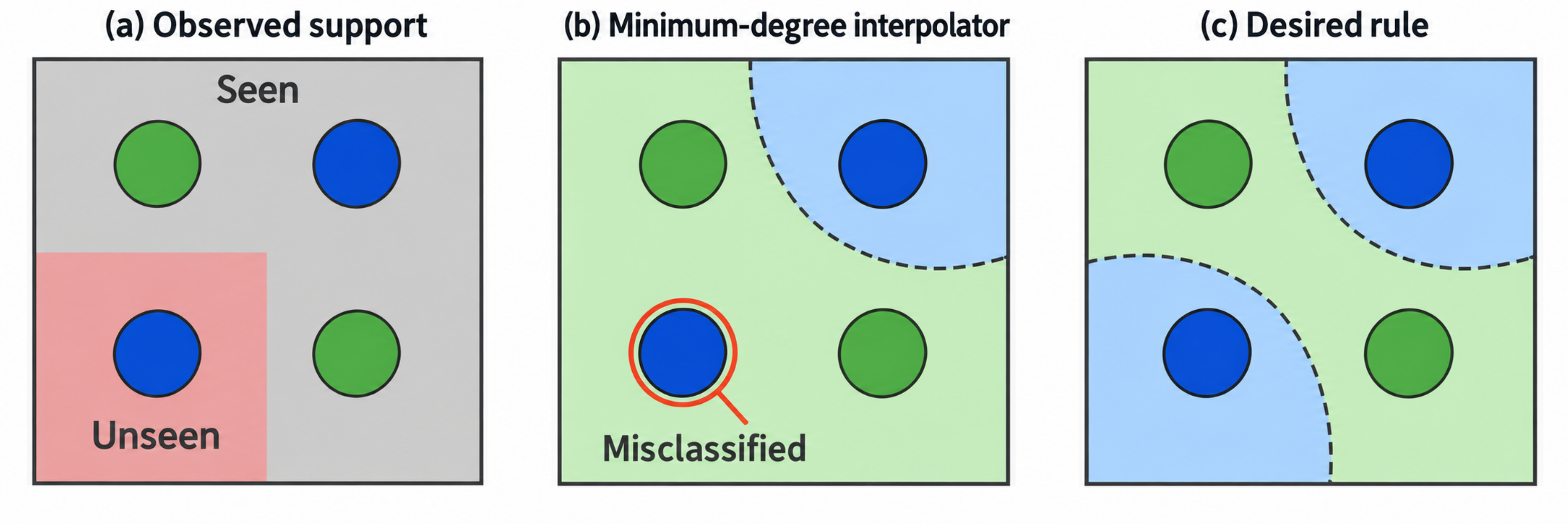}
\caption{Generalization on the Unseen for the parity problem. (a) The model is trained on the three seen points. (b) The model converges to a simple decision rule (minimum degree interpolator) which perfectly classifies seen points but which misclassifies an unseen point requiring higher order interaction. (c) The correct decision rule that is not learned when converging to a MDI solution.}
\label{fig:seen_unseen_predictions-1}
\end{figure}

\section{Related Works}
\label{sec:relatedwork}
\subsection{Neural Network Initialization and Inductive Biases}
Initialization methods have been widely explored in the deep learning community. Traditionally, they are designed to stabilize the training process, as deep networks often suffer from vanishing or exploding gradients \citep{pmlr-v9-glorot10a, he2015delving}. More recent work has focused on how initialization can influence learning efficiency. For instance, \citep{xu2023initializing,samragh2024scaling,samragh2023weight} proposes using pretrained weights from large models to initialize smaller models. Similarly, \citep{trockman2023mimetic} suggests transferring structural patterns from well-trained transformers as initialization, improving performance across a variety of tasks. Initialization is also widely applied for optimization stability, often by analyzing signal propagation under assumptions on pseudo-inputs such as Gaussian-distributed features \citep{sitzmann2020implicit, lindell2022bacon, kumar2017weight}.

Inductive bias, which characterizes the generalization behavior of machine learning algorithms, has long been a central topic in the field. The inductive bias of transformers has been extensively studied, with many approaches seeking to explore the limits of extrapolation, generalization, and encoding priors through training\citep{abbe2024generalization,wanganti,zhong2024algorithmic,pushkinminimal,davidovich2026algorithmic, teney2026can, shinnick2026learnimagesproceduralwarmup}. More recently, \citep{zheng2024structured,zheng2025structured} propose \textit{structured initialization} to address the lack of inductive bias in standard transformer initialization, which typically yields unstructured attention maps. 
Their approach injects CNN-like locality into ViTs by solving for query and key matrices such that $\mathrm{Softmax}(QK^\top/\sqrt{d})$ approximates convolutional impulse patterns. 
This imposes spatial structure at initialization without modifying the architecture, leading to improved performance on small- and medium-scale datasets while maintaining competitive results at scale. However, this induced structure is implicitly encoded through query-key interactions, which may not preserve the desired inductive bias during training.



\subsection{Attention Masks in Transformers}
Attention masks, as a flexible and easily implemented component, play an important role in transformer architectures. The causal mask is a feature in the original autoregressive architecture, preventing the model from accessing future tokens and thereby avoiding shortcut solutions that could harm inference performance \citep{pmlr-v235-yin24a,pei2025rethinking, vaswani2017attention}. Beyond causality, masks are used to restrict unwanted interactions in attention, particularly when data is limited in both size and dimensionality. This is especially important in vision models and tasks requiring length generalization \citep{zaheer2020big,beltagy2020longformer,tay2020sparse,child2019generating,cheng2022masked, Cheng_2022_CVPR, Pang_2019_ICCV, qihang2022}. Moreover, masking is a key tool for improving computational efficiency, as the quadratic $O(N^2)$ complexity of self-attention can be reduced by limiting the number of token interactions \citep{ho2019axial,roy2021efficient,yuan2025native}. Related to our setting, \citet{duan2024attention} use fixed, task-specific attention masks to enable length generalization for transformers applying arithmetic rules to longer inputs. In contrast, our work explore whether a learnable mask can enable extrapolation in the GOTU setting by encoding a persistent prior towards a higher degree interaction. Additionally, we demonstrate that how inductive biases are encoded in transformers matters, with alternative approaches to encode biases such as structured initialization failing to preserve desired attention patterns over training (see \Cref{fig:seen_unseen_predictions}). 










\section{Preliminary}
\label{sec:pre}
\paragraph{Self-Attention} Self-attention is a core component of transformer models \citep{vaswani2017attention}. Given an input sequence $X\in\mathbb{R}^{N\times D}$, self-attention is parameterized by three learnable projection matrices: query $W_Q\in\mathbb{R}^{D\times d}$, key $W_K\in\mathbb{R}^{D\times d}$, and value $W_V\in\mathbb{R}^{D\times d}$. The output of a single attention head is defined as
\begin{equation}
A(X)
=
\mathrm{softmax}\!\left(
\frac{XW_QW_K^\top X^\top}{\sqrt d}
\right)XW_V,
\label{eq:self_attn}
\end{equation}
where the softmax operator acts row-wise on the attention score matrix. In practice, multiple such heads are used in parallel and concatenated to form a multi-head attention layer.

\paragraph{Minimum-Degree Interpolator}
The notion of a minimum-degree interpolator (MDI) \citep{abbe2024generalization} characterizes an inductive bias where neural networks prefer low-degree structure when fitting data. 
For Boolean functions $f:\{\pm1\}^d \to \mathbb{R}$ with Fourier expansion
\[
f(x)=\sum_{S\subseteq[d]} \hat f(S)\chi_S(x), \quad \chi_S(x)=\prod_{i\in S}x_i,
\]
where $\hat f(S)$ is the coefficient for each polynomial and $[d]$ is the set $\{0,1,\cdots, d\}$. Multiple functions may interpolate the observed data on $\Omega_{\mathrm{seen}}$. 
MDI selects the solution that explains the data using the lowest-order terms, favoring low-degree structure over higher-order interactions. 
As shown in \citep{abbe2024generalization}, transformers have a tendency to favor these low-degree solutions over the true higher-order functions for Boolean tasks such as 2Parity and 3-Bit Majority, representing a substantial limitation in the ability for transformers to perform extrapolation within logic problems.


\section{Methods}
\label{sec:method}
\begin{figure}[t]
    \centering
    \includegraphics[width=\linewidth]{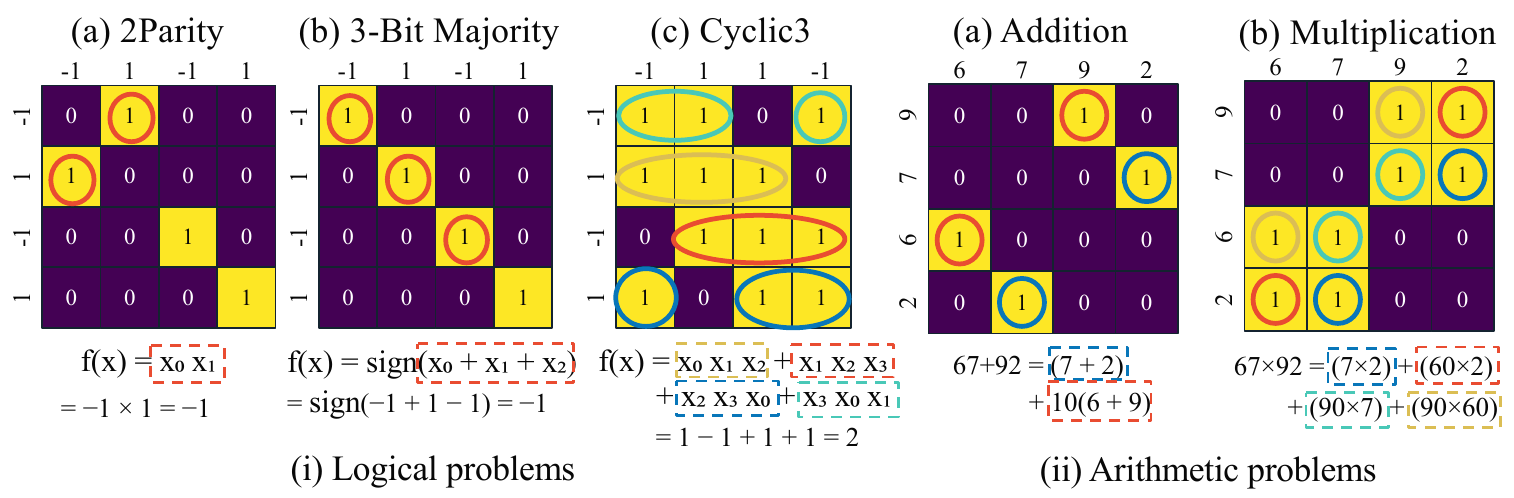}

    \caption{
    Visualization of task-specific mask design. 
    \textbf{(i) Logical problems.} 
    (a) 2Parity encodes pairwise interaction between the first two tokens; 
    (b) 3-Bit Majority emphasizes aggregation over the first three tokens; 
    (c) Cyclic3 captures cyclic triple interactions. 
    \textbf{(ii) Arithmetic problems.} 
    (a) Addition enforces digit-wise interactions between aligned positions; 
    (b) Multiplication encodes structured cross-digit interactions corresponding to partial products. 
    The masks define the interaction structure underlying each task, aligning attention with the functional decomposition.
    }
    \label{fig:mask_all}
\end{figure}
\subsection{Encoding Inductive Biases through a Logical Mask}

For many structured tasks such as image classification and logical problems, prior knowledge about token interactions is often available. This knowledge can be expressed as a correlation map over tokens, indicating which interactions should be encouraged or suppressed. Structured initialization methods \citep{zheng2025structured,zheng2024structured,trockman2023mimetic} aim to encode such prior knowledge into self-attention, typically by shaping the initial attention patterns through the query-key matrices.

In our approach, we directly encode prior structure as a learnable mask $M \in \mathbb{R}^{N \times N}$. 
The mask is initialized from manually constructed patterns that reflect the underlying problem structure. For example, \Cref{fig:mask_all} illustrates how the desired inductive biases for Boolean and arithmetic tasks can be encoded as attention maps, where only relevant tokens (e.g., aligned digits or interacting variables) are allowed to attend. Entries corresponding to desired interactions are assigned active values, while undesired interactions are assigned smaller (more negative) values to suppress them. Through the transformation $\log\sigma(M)$, this induces a soft, learnable bias on the attention logits, enabling the model to retain the prescribed structure while adapting during training. The full formula is $
\mathrm{softmax}\!\left(
\frac{XW_QW_K^\top X^\top}{\sqrt d}
+
\log\sigma(M)
\right),
$
where $\sigma(\cdot)$ denotes the sigmoid function. 
Since $\log\sigma(M)\le 0$, this selectively down-weights undesired interactions while preserving the standard attention mechanism. 
During training, $M$ remains learnable and is optimized jointly with the model parameters, allowing the encoded structural prior to persist throughout optimization.

\subsection{Principles for Designing Logical Masks}
Based on the previous mask methods \citep{vaswani2017attention, duan2024attention}, which blocks access to certain tokens to prevent spurious correlations and shortcut learning, 
we hypothesize that properly designed masks can similarly restrict irrelevant interactions in other tasks, particularly in low-data regimes where learning correlations is difficult, thereby improving generalization. 
To this end, we outline principles for mask-based initialization: the key idea is to encode prior knowledge about which token interactions are relevant for the task, while suppressing spurious correlations.
As a simple example, consider the 3-bit majority task and its MDI solution. The non-MDI solution can be expressed as a global summation followed by a sign function, whereas the MDI solution relies on pairwise interaction terms between inputs. 
Therefore, as shown in \Cref{fig:mask_all}, we suppress such interaction terms by restricting attention, preventing the model from adopting the MDI solution and more uniform distributed attentions. This induces a structured attention pattern aligned with the underlying 3-Bit Majority algorithm. We visualize the corresponding mask patterns for different tasks in \Cref{fig:mask_all}, 
and provide more detailed design principles in Supplementary Materials \Cref{app:init_mask}.

\section{Experiments}
\label{sec:exp}
\subsection{Boolean GOTU Problems}

\begin{table}[!h]
      \centering
      \caption{
      Test performance across Boolean tasks, reported as mean $\pm$ sample standard deviation over five random seeds. SVD Init and QK Opt denote task-pattern initialization via SVD  \citep{zheng2025structured} and query-key optimization \citep{zheng2024structured}, respectively.   Bold indicates the best mean at the displayed precision, including ties.
      }
      \label{tab:boolean_results}

      \small
      \setlength{\tabcolsep}{6pt}
      \renewcommand{\arraystretch}{1.0}
      \resizebox{\linewidth}{!}{
      \begin{tabular}{llcccc}
      \toprule
      Task & Metric
      & Default
      & Structured (SVD)
      & Structured (Opt)
      & Logical Mask (Ours)\\
      \midrule

      \multirow{3}{*}{2Parity}
      & Loss $\downarrow$
      & $1.41 \pm 0.25$
      & $1.28 \pm 0.09$
      & $1.36 \pm 0.27$
      & $\mathbf{0.01 \pm 0.02}$ \\
      & Seen Acc. (\%) $\uparrow$
      & $\mathbf{100.00 \pm 0.00}$
      & $\mathbf{100.00 \pm 0.00}$
      & $\mathbf{100.00 \pm 0.00}$
      & $\mathbf{100.00 \pm 0.00}$ \\
      & Unseen Acc. (\%) $\uparrow$
      & $0.00 \pm 0.00$
      & $0.00 \pm 0.00$
      & $0.00 \pm 0.00$
      & $\mathbf{100.00 \pm 0.00}$ \\
      \cmidrule(lr){1-6}

      \multirow{3}{*}{Cyclic3}
      & Loss $\downarrow$
      & $4.00 \pm 2.48$
      & $0.57 \pm 1.01$
      & $4.12 \pm 2.24$
      & $\mathbf{0.01 \pm 0.00}$ \\
      & Seen Acc. (\%) $\uparrow$
      & $93.59 \pm 4.19$
      & $94.95 \pm 6.82$
      & $95.70 \pm 2.27$
      & $\mathbf{99.24 \pm 0.00}$ \\
      & Unseen Acc. (\%) $\uparrow$
      & $4.14 \pm 5.07$
      & $59.06 \pm 41.20$
      & $0.19 \pm 0.41$
      & $\mathbf{99.24 \pm 0.00}$ \\
      \cmidrule(lr){1-6}

      \multirow{3}{*}{3-Bit Majority}
      & Loss $\downarrow$
      & $0.35 \pm 0.22$
      & $0.23 \pm 0.21$
      & $0.29 \pm 0.26$
      & $\mathbf{0.01 \pm 0.00}$ \\
      & Seen Acc. (\%) $\uparrow$
      & $\mathbf{100.00 \pm 0.00}$
      & $\mathbf{100.00 \pm 0.01}$
      & $\mathbf{100.00 \pm 0.00}$
      & $\mathbf{100.00 \pm 0.00}$ \\
      & Unseen Acc. (\%) $\uparrow$
      & $60.11 \pm 22.30$
      & $64.15 \pm 20.70$
      & $70.03 \pm 27.36$
      & $\mathbf{100.00 \pm 0.00}$ \\
      \cmidrule(lr){1-6}

      \multirow{2}{*}{ThreeSym}
      & Seen Loss ($10^{-4}$) $\downarrow$
      & $\mathbf{3.30 \pm 0.42}$
      & $7.49 \pm 3.09$
      & $3.98 \pm 1.24$
      & $3.31 \pm 1.44$ \\
      & Unseen Loss $\downarrow$
      & $14.98 \pm 0.26$
      & $14.25 \pm 1.10$
      & $15.11 \pm 0.41$
      & $\mathbf{5.44 \pm 0.21}$ \\

      \bottomrule
      \end{tabular}}
  \end{table}
  
We evaluate our method on four Boolean tasks under the Generalization on the Unseen (GOTU) setting \cite{abbe2024generalization} which are detailed described in Supplementary \Cref{sec:GOTU}. These tasks are important as they reflect the extrapolation and inductive bias properties of transformers. Performance on these datasets provides a useful lens for evaluating whether a properly designed initialization, with the correct correlation patterns, can alter extrapolation behavior. Visualization and explanation of the initialization pattern can be found at \Cref{app:init_mask}.

\begin{figure}[!t]
\centering
\includegraphics[width=\linewidth]{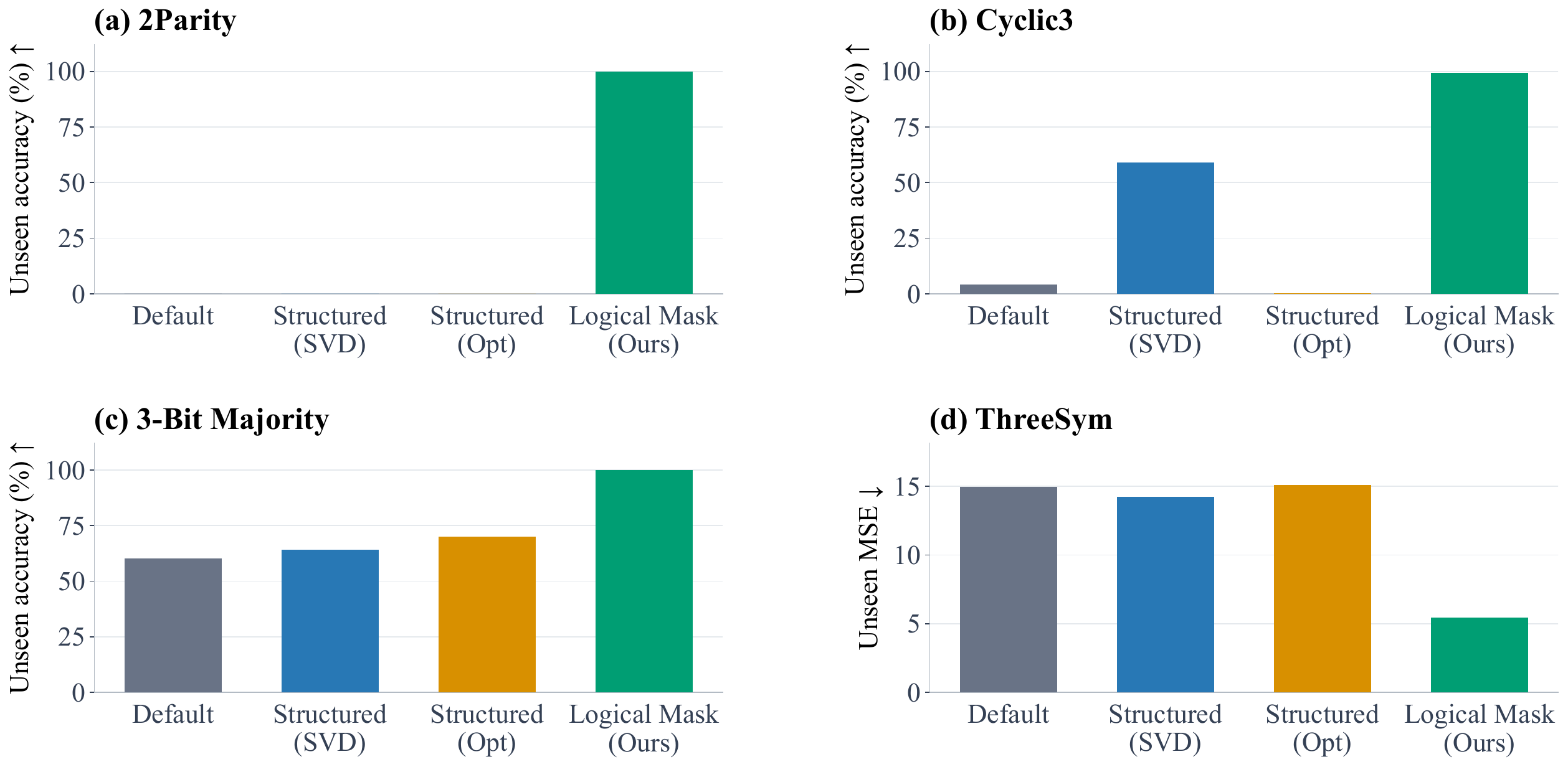}
\caption{
Test performance on Boolean reasoning tasks under an extrapolation setting. Our logical mask consistently achieves lower loss and higher accuracy compared to default and QK-based initialization methods, indicating improved extrapolation. Accuracy is not reported for ThreeSym due to its continuous outputs.
}
\label{fig:boolean_results}
\end{figure}

From \Cref{fig:boolean_results} and \Cref{tab:boolean_results}, we observe that initialization can indeed influence the inductive bias of transformers, as the GOTU setting explicitly evaluates extrapolation beyond the training domain. Moreover, even when using the same structural pattern, initialization through the query-key matrices tends to quickly forget the imposed prior during optimization, which leads to poor generalization on unseen data. In contrast, mask-based initialization preserves the prior and enables better extrapolation.

\subsection{Arithmetic Problems}
Arithmetic problems such as addition and multiplication follow simple and clear rules, but they are still difficult for neural networks to learn. For example, addition only requires learning basic operations such as adding single digits and handling carries. Once these rules are learned, they can be reused to solve many new examples. Therefore, we would expect a model to learn arithmetic from a relatively small number of training examples. However, neural networks often need much more data than expected, even to generalize to new examples with the same input format \citep{trask2018neural}.
 
This makes arithmetic a natural test for sample efficiency: if a model can better use the underlying rules, it should require fewer training examples to generalize. We test whether using a logical mask improves in-distribution generalization in low-data settings. For both addition and multiplication, we enumerate all possible input pairs ($10^6$ combinations) and use only a fraction of them for training, with the remaining examples used for testing. We evaluate 3-by-3 digit decimal addition and multiplication, where all operands contain exactly three digits (including leading zeros), and the transformer predicts the full output sequence in a single forward pass. More details are provided in \Cref{sec:arith_def}.


\begin{figure}[!ht]
\centering
    \includegraphics[width=\linewidth]{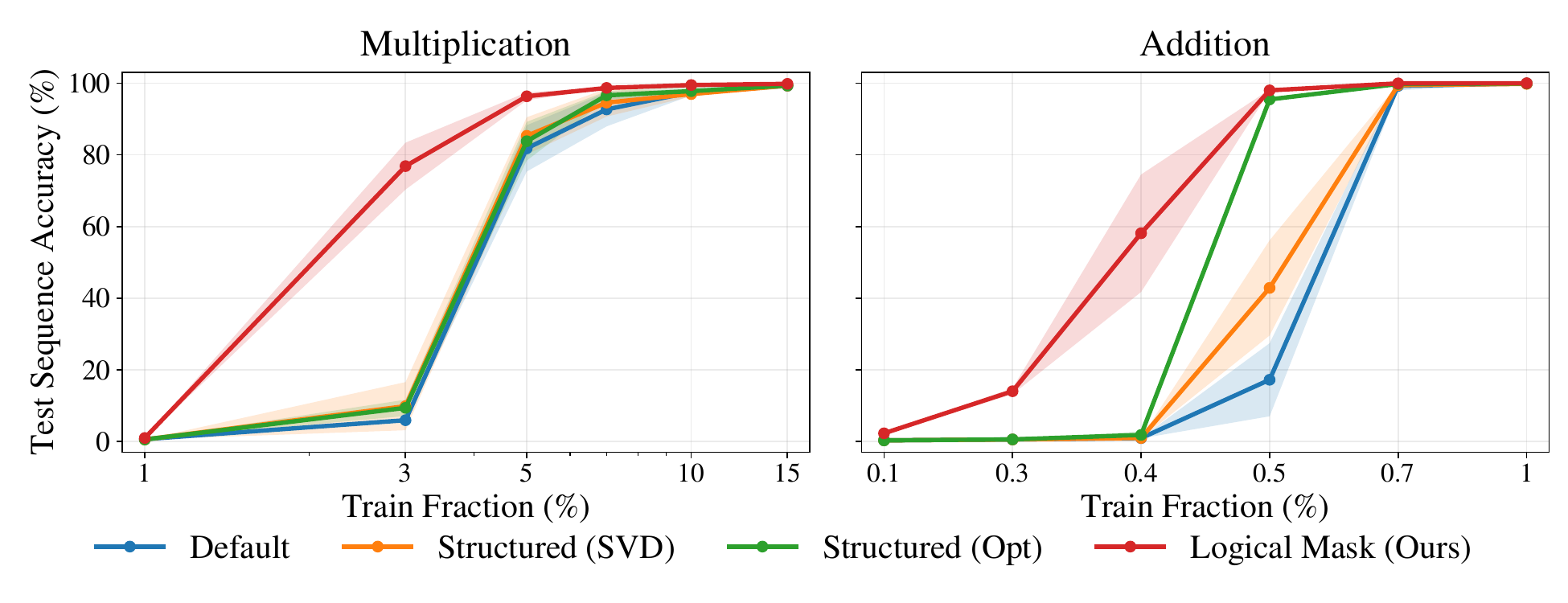}
\caption{
Test sequence accuracy under varying training fractions for multiplication (left) and addition
(right). Logical mask improves data efficiency, achieving higher accuracy in low-data
regimes.
}
\label{fig:data_efficiency}
\end{figure}

\begin{table}[!ht]
\centering
\caption{
Test accuracy on arithmetic tasks, reported as mean $\pm$ sample standard deviation across five
random seeds. Higher is better.
}
\label{tab:arithmetic}
\resizebox{\linewidth}{!}{
\begin{tabular}{l|cc|cc}
\toprule
& \multicolumn{2}{c|}{Multiplication (Train frac = 3\%)}
& \multicolumn{2}{c}{Addition (Train frac = 0.4\%)} \\
\cmidrule(lr){2-3} \cmidrule(lr){4-5}
Method
& Seq Acc (\%) $\uparrow$ & Digit Acc (\%) $\uparrow$
& Seq Acc (\%) $\uparrow$ & Digit Acc (\%) $\uparrow$ \\
\midrule
Default
& $6.51 \pm 0.81$ & $71.61 \pm 0.58$
& $1.01 \pm 0.19$ & $39.49 \pm 1.21$ \\

Structured (SVD)
& $10.04 \pm 4.76$ & $74.28 \pm 3.13$
& $1.10 \pm 0.24$ & $40.23 \pm 1.60$ \\

Structured (Opt)
& $9.74 \pm 1.69$ & $73.96 \pm 1.48$
& $2.11 \pm 0.90$ & $44.68 \pm 3.83$ \\

Logical Mask (Ours) 
& $\mathbf{77.51 \pm 7.51}$ & $\mathbf{95.76 \pm 1.51}$
& $\mathbf{60.37 \pm 16.00}$ & $\mathbf{87.70 \pm 5.76}$ \\
\bottomrule
\end{tabular}
}
\end{table}





As shown in \Cref{tab:arithmetic}, The logical mask consistently outperforms other methods on both multiplication and addition tasks. 
The is most pronounced in low-data regimes (e.g., addition with 0.4\% of smamples), where other methods fail to generalize. As shown in \Cref{fig:data_efficiency}, although all methods improve with more data, the logical mask attains better generalization with fewer training samples. This provides evidence that a logical mask can accelerate learning an underlying algorithmic structure.

\subsection{Image Classification}
It is known that spatial priors for vision tasks can be encoded in attention through structured initialization \citep{zheng2025structured}. 
Following \citep{zheng2025structured}, we adopt an impulse filter pattern to initialize our logical masks (see Supplementary Figure \Cref{app:impulse} for visualization of the desired pattern). 
 In this section, we evaluate three initialization methods on CIFAR10, CIFAR100, and ImageNet100 \citep{imagenetrussakovsky, Krizhevsky09learningmultiple}. For comparison, we adopt the implementation from \citep{zheng2025structured}, which provides the most up-to-date codebase. 
We include these benchmarks to demonstrate that the proposed patterns extend beyond toy settings and remain effective in practical scenarios. 
Moreover, these datasets align with the primary evaluation domain of prior structured initialization methods. We compare against mimetic initialization \citep{trockman2023mimetic}, which initializes attention close to an identity pattern with noise, reflecting a common structure observed in trained transformers, and the Impulse method \citep{zheng2025structured}, which initializes each attention head to approximate a $3 \times 3$ impulse filter, resulting in a diagonal-like attention pattern. 
Both approaches inject locality as an inductive bias, which has been shown to improve generalization, particularly in low-data regimes.

\begin{table}[!h]
  \centering
  \caption{Best Top-1 accuracy (\%) on vision benchmarks using
  ViT-Tiny. Mean $\pm$ sample standard deviation
  on three random seeds. For each seed, we report
  the highest accuracy during training.}
  \label{tab:cv_results}
  \resizebox{0.9\linewidth}{!}{
  \begin{tabular}{lcccc}
  \toprule
  Dataset & Default & Mimetic & Impulse & Logical Mask (Ours)\\
  \midrule
  CIFAR10
  & $93.93 \pm 0.11$
  & $94.98 \pm 0.26$
  & $95.03 \pm 0.20$
  & $\mathbf{95.95 \pm 0.12}$ \\
  CIFAR100
  & $77.26 \pm 0.22$
  & $79.01 \pm 0.06$
  & $79.62 \pm 0.35$
  & $\mathbf{82.06 \pm 0.14}$ \\
  ImageNet100
  & $83.47 \pm 0.13$
  & $84.55 \pm 0.14$
  & $83.51 \pm 0.95$
  & $\mathbf{85.83 \pm 0.56}$ \\
  \bottomrule
  \end{tabular}
  }
  \end{table}
As shown in \Cref{tab:cv_results}, our logical mask consistently achieves the best performance, with higher accuracy than baseline and structured methods, indicating improved optimization and generalization.
Additional visualizations are provided in \Cref{app:trained_mask,app:trained_img}.

\begin{figure}[!ht]
    \centering
    \includegraphics[width=1\linewidth]{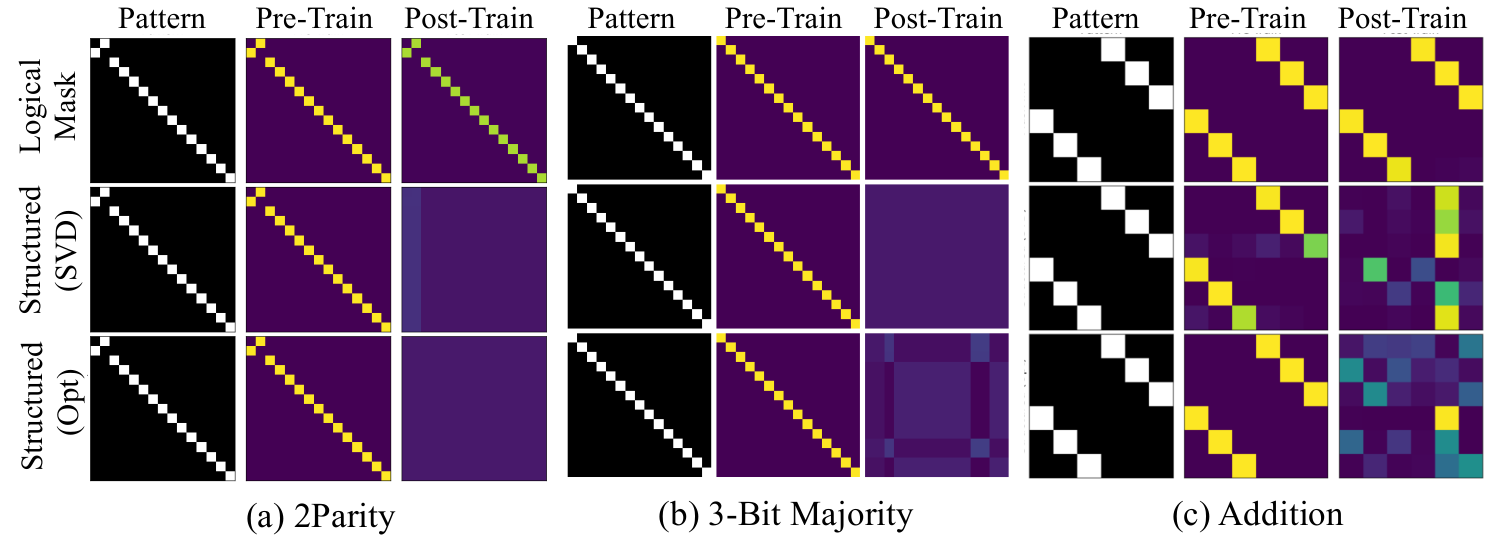}
    \caption{
    Comparison of attention patterns before and after training for Head~0 at Layer~0.
    Our logical mask preserves the target interaction structure more closely after training, whereas QK-based structured initializations are not preserved despite closely matched initial attention patterns.
    }
    \label{fig:comparison_pattern}
\end{figure}

\section{Persistence of Prior Knowledge during Optimization}

In this section, we examine why structural priors encoded through additive masks are more persistent during optimization than those encoded through structured QK-initialization by examining gradient effects of the self-attention. The key distinction between $QK$ initialization and our logical mask is that a mask represents the prior through an additive pathway in the attention logits, whereas a $QK$-based initialization is embedded in the query and key parameters that can overfit to seen examples during training. Consequently, changes to these representations can alter the encoded attention structure even when the desired pattern is well matched at initialization. To characterize the difference between these two pathways, we analyze the local sensitivity of the attention probabilities with respect to queries, keys, and mask parameters. 



\begin{theorem}
\label{thm:masked_attention_jacobian}
Consider a single attention row with query $q\in\mathbb{R}^d$, keys $\{k_j\}_{j=1}^N$, and mask parameters $\{m_j\}_{j=1}^N$. Define
\[
p_j=
\frac{\exp(q^\top k_j/\sqrt d+\log\sigma(m_j))}
{\sum_{\ell=1}^N\exp(q^\top k_\ell/\sqrt d+\log\sigma(m_\ell))}.
\]
Then, for any $i,j\in\{1,\ldots,N\}$,
\begin{align}
\frac{\partial p_i}{\partial q}
&=
\frac{p_i}{\sqrt d}
\left(
k_i-\sum_{\ell=1}^N p_\ell k_\ell
\right),
\\
\frac{\partial p_i}{\partial k_j}
&=
\frac{1}{\sqrt d}
p_i(\delta_{ij}-p_j)q,
\\
\frac{\partial p_i}{\partial m_j}
&=
p_i(\delta_{ij}-p_j)\sigma(-m_j).
\end{align}
\end{theorem}

\begin{corollary}
\label{cor:masked_attention_grad_norm}
Under the setting of Theorem~\ref{thm:masked_attention_jacobian}, suppose
$\|q\|_2\le B_q$ and $\|k_j\|_2\le B_k$ for all $j$. Then
\begin{align}
\left\|\frac{\partial p_i}{\partial q}\right\|_2
&\le
\frac{2p_iB_k}{\sqrt d},
\\
\left\|\frac{\partial p_i}{\partial k_j}\right\|_2
&\le
\frac{p_i|\delta_{ij}-p_j|B_q}{\sqrt d}
\le
\frac{p_iB_q}{\sqrt d},
\\
\left|\frac{\partial p_i}{\partial m_j}\right|
&\le
p_i|\delta_{ij}-p_j|\sigma(-m_j)
\le
p_i .
\end{align}
\end{corollary}


The Corollary above highlight a difference in local sensitivity: query-key gradients depend on the feature scales $\|q\|_2$ and $\|k_j\|_2$, whereas mask gradients do not and are further moderated by $\sigma(-m_j)\leq 1$. For strongly encouraged interactions, large positive $m_j$ reduces this sensitivity, while strongly suppressed interactions typically have small attention probability $p_j$. These properties provide a gradient-level mechanism consistent with the greater empirical persistence of mask-based priors, rather than implying universally smaller optimizer updates. Proofs are provided in \Cref{app:proof2}. We also verify this persistence directly by tracking the attention pattern during training. As shown in \Cref{fig:comparison_pattern} and \Cref{sec:empirical_persistency}, the additive-mask parameterization retains the target interaction structure substantially more closely, whereas the attention patterns produced by QK-based structured initialization drift away from their initialized structure. Importantly, this difference is observed even when the two parameterizations produce closely matched attention patterns at initialization, suggesting that the effect is not simply due to better initial pattern fitting. The resulting difference in persistence is accompanied by substantially better extrapolation on the Boolean GOTU tasks.

\begin{figure}[!t]
\centering
\includegraphics[width=\textwidth]{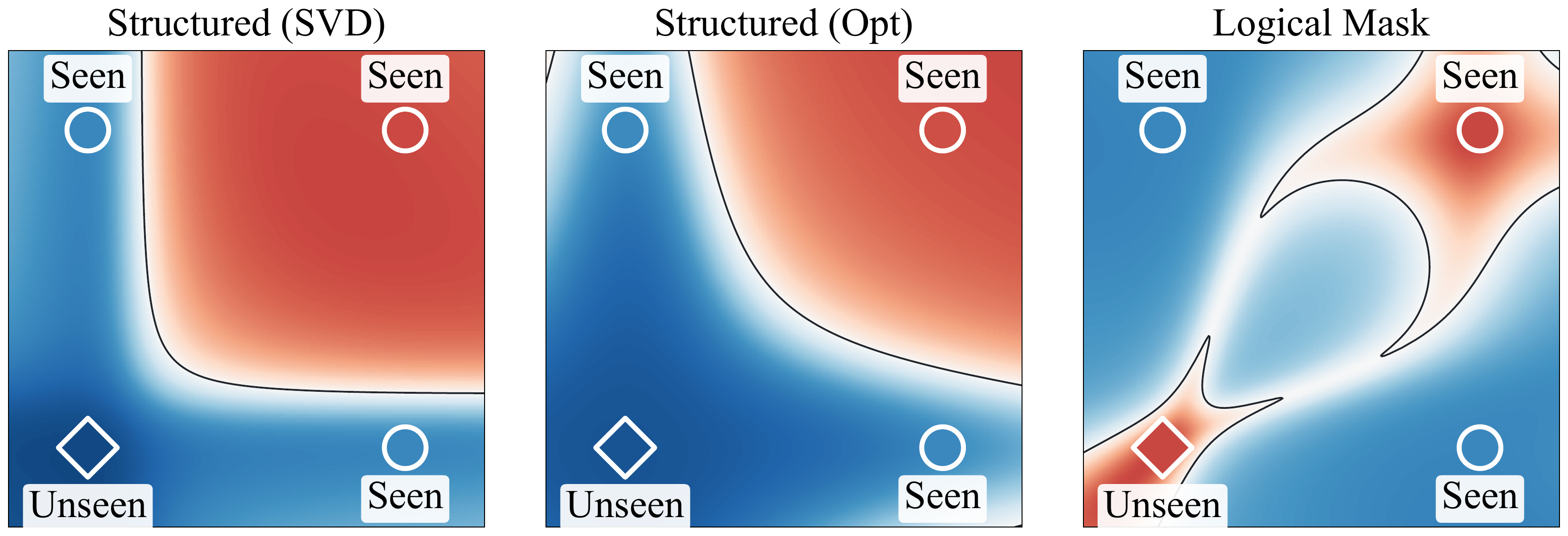}
\caption{Decision boundaries on the GOTU 2Parity problem. While all three methods start with the same structured attention pattern, both of the structured initalization methods (SVD \citep{zheng2025structured} and Opt \citep{zheng2024structured}) converge to a MDI solution. In contrast, the logical mask is more persistent during training enabling it to achieve the desired higher-order interaction for extrapolation.  }
\label{fig:seen_unseen_predictions}
\end{figure}

\section{Conclusion and Limitations}
\label{sec:concl}
In conclusion, we have shown that applying an appropriately designed attention mask is able to encode a logical decision rule in transformers. As this mask encodes a persistent desired token interaction, this is able to avoid converging to the simple interpolative solution that present a difficult when training transformers on Boolean problems in a GOTU setting. Furthermore, we have shown that although a similar inductive bias can be encoded with a structured QK-based initialization, this initialization does not persist during the course of training. As a result, we have shown that where we encode an inductive bias matters, and that a bias encoded through a logical mask can be maintained with benefits across multiple tasks. Empirically, we have demonstrated that a logical mask is sufficient to enable near perfect performance on Boolean GOTU tasks, and that appropriately constructed masks can improve the training efficiency for arithmetic problems and image classification. Our work provides intuition on encoding inductive biases in self-attention and demonstrates the potential for logical masks to embed persistent decision-rules in transformers.

\paragraph{Limitations.}
Our approach assumes access to a meaningful interaction prior, which is natural in structured domains such as logic, arithmetic, and vision, but less clear in domains such as language. This is in effect an oracle assumption which allows us to demonstrate that useful inductive priors may be encoded within transformers using masks, but this may not be known in advance for all tasks. Designing effective masks may require knowledge of domain characteristics, and learning priors from data or pretrained models for general problems remains an important direction for future work.
\newpage


\bibliographystyle{iclr2027_conference}
\bibliography{iclr2027_conference}

\newpage
\appendix
\setcounter{theorem}{0}
\setcounter{corollary}{0}
\section{Task Definitions}
\subsection{Boolean GOTU Task Definitions}
\label{sec:GOTU}
We consider four generalization-on-the-unseen (GOTU) Boolean tasks introduced by \citet{abbe2024generalization}. For each task, let $f$ denote the ground-truth target function and $\tilde f$ denote the minimum-degree interpolator (MDI), i.e., the lowest-degree function that agrees with $f$ on the seen domain. These tasks are constructed such that $\tilde f$ matches $f$ on all training examples but has lower degree or a simpler interaction structure. They therefore provide a controlled setting for studying whether transformers follow the MDI bias or preserve the higher-order interactions of the target function. 

\paragraph{2Parity.}
$f(x)=x_0x_1$, seen domain $x_0=1$ or $x_1=1$ (i.e., excluding $(-1,-1)$). On this domain, $x_0x_1 = x_0 + x_1 - 1$, so $\tilde f(x)=x_0 + x_1 - 1$ (degree $2\to1$).

\paragraph{Cyclic3.}
$f(x)=\sum_{i=0}^{14} x_i x_{i+1} x_{i+2}$, seen domain excludes $(x_0,x_1,x_2)=(-1,-1,-1)$. The first term satisfies $x_0x_1x_2 = x_0x_1 + x_1x_2 + x_2x_0 - x_0 - x_1 - x_2 + 1$, so $\tilde f$ replaces this cubic term by a quadratic one (degree $3\to2$).

\paragraph{3-Bit Majority.}
$f(x)=\mathrm{Maj}(x_0,x_1,x_2)=\frac{x_0+x_1+x_2-x_0x_1x_2}{2}$, seen domain $x_0=1$ or $x_1=1$. On this domain, the cubic term $x_0x_1x_2$ can be rewritten into pairwise terms, yielding $\tilde f(x)=\frac{x_0+x_1+2x_2-x_0x_2-x_1x_2}{2}$ (degree $3\to2$).

\paragraph{ThreeSym.}
$f(x)=x_0x_1 - 1.25\,x_1x_2 + 1.5\,x_2x_0$, seen domain $x_0x_1x_2=1$. On this domain, $x_0x_1\to x_2$, $x_1x_2\to x_0$, $x_2x_0\to x_1$, so $\tilde f(x)=x_2 - 1.25\,x_0 + 1.5\,x_1$ (degree $2\to1$).

\subsection{Arithmetic Definitions}
\label{sec:arith_def}
\paragraph{Addition}
We evaluate 3-by-3 digit decimal addition in an in-distribution setting, where all operands have exactly three digits. 
The transformer predicts the full output sequence in a single forward pass.

\paragraph{Multiplication}
We evaluate 3-by-3 digit decimal multiplication under the same setting. 
Compared to addition, multiplication is more challenging as it requires multiple intermediate digit-wise multiplications, carry operations, and the accumulation of partial sums.
\section{Ablation Study}

\subsection{Controlling for Additional Mask Parameters}
\label{sec:additional_mask_parameters}

To test whether the improvement comes simply from adding extra mask parameters, we compare logical mask with three controls: a randomly initialized learnable mask, an all-zero initialized learnable mask, and a fixed logical mask. The random and all-zero masks have the same number of learnable parameters as logical mask but contain no task-specific structural prior. For the all-zero mask, the log-sigmoid parameterization adds the same constant to every attention logit, so the initial attention distribution is unchanged. The fixed mask encodes the same structural prior as logical mask but is not updated during training, allowing us to distinguish the effect of the prior itself from that of its learnability.

\begin{table}[!h]
    \centering
    \caption{
    Effect of logical mask on Boolean GOTU tasks. Results are reported as mean $\pm$ sample standard deviation over five random seeds. Test loss is the overall test MSE; accuracy (\%) is reported separately on seen and unseen test inputs.
    }
    \label{tab:additional_mask_parameters}

    \small
    \setlength{\tabcolsep}{4pt}
    \begin{tabular}{llccc}
    \toprule
    Task & Method & Loss $\downarrow$
    & Seen Acc $\uparrow$ & Unseen Acc $\uparrow$ \\
    \midrule
    2Parity
    & Default
    & $1.4111 \pm 0.2484$
    & $100.00 \pm 0.00$ & $0.00 \pm 0.00$ \\
    & Random Mask
    & $1.4196 \pm 0.1913$
    & $100.00 \pm 0.00$ & $0.00 \pm 0.00$ \\
    & All-zero
    & $1.3927 \pm 0.2614$
    & $100.00 \pm 0.00$ & $0.00 \pm 0.00$ \\
    & Fixed Mask
    & $0.0294 \pm 0.0338$
    & $100.00 \pm 0.00$ & $80.00 \pm 44.72$ \\
    & Logical Mask
    & $0.0138 \pm 0.0172$
    & $100.00 \pm 0.00$ & $100.00 \pm 0.00$ \\
    \midrule
    Cyclic3
    & Default
    & $3.9964 \pm 2.4791$
    & $93.59 \pm 4.19$ & $4.14 \pm 5.07$ \\
    & Random Mask
    & $3.6310 \pm 1.8705$
    & $94.22 \pm 2.76$ & $0.93 \pm 2.07$ \\
    & All-zero
    & $4.1167 \pm 1.7160$
    & $94.96 \pm 3.39$ & $1.65 \pm 2.41$ \\
    & Fixed Mask
    & $0.0100 \pm 0.0010$
    & $99.24 \pm 0.00$ & $99.24 \pm 0.00$ \\
    & Logical Mask
    & $0.0098 \pm 0.0012$
    & $99.24 \pm 0.00$ & $99.24 \pm 0.00$ \\
    \midrule
    3-Bit Majority
    & Default
    & $0.3459 \pm 0.2223$
    & $100.00 \pm 0.00$ & $60.11 \pm 22.30$ \\
    & Random Mask
    & $0.3258 \pm 0.1914$
    & $100.00 \pm 0.00$ & $59.97 \pm 22.39$ \\
    & All-zero
    & $0.3014 \pm 0.2675$
    & $100.00 \pm 0.00$ & $70.25 \pm 27.16$ \\
    & Fixed Mask
    & $0.0128 \pm 0.0053$
    & $100.00 \pm 0.00$ & $100.00 \pm 0.00$ \\
    & Logical Mask
    & $0.0148 \pm 0.0012$
    & $100.00 \pm 0.00$ & $100.00 \pm 0.00$ \\
    \bottomrule
    \end{tabular}
  \end{table}

As shown in \Cref{tab:additional_mask_parameters}, neither the random mask nor the all-zero mask consistently improves over the Default transformer, especially on unseen inputs. In contrast, both Fixed Mask and logical mask, which encode the correct structural prior, substantially improve extrapolation. Logical mask achieves near-perfect unseen accuracy on the three discrete tasks and also greatly reduces the test loss on ThreeSym. These results show that the gains come from the structural information encoded in the mask, rather than from simply adding more learnable parameters.

\subsection{Why Not Just Fix the Mask?}
Another natural question is whether learnability is necessary, since fixed masks can already provide strong inductive biases on structured tasks such as Boolean and arithmetic problems. However, in more realistic settings, the available prior may only approximately reflect the underlying interaction structure. In such cases, learnability becomes important because it allows the model to refine the initial prior during training rather than being constrained to a fixed interaction pattern.

The results in \Cref{tab:ablation_cv} illustrate this distinction. Fixed masks improve over the default Transformer on all three vision benchmarks, showing that even a fixed structural prior can be beneficial. The gain is particularly pronounced on CIFAR100, where the fixed mask improves Top-1 accuracy from $77.26\%$ to $79.84\%$ (+2.58 points). However, the benefit of the fixed mask is much smaller on CIFAR10 and ImageNet100, where it improves over the default model by only +0.67 and +0.40 points, respectively. In contrast, the learnable logical mask achieves substantially larger improvements of +2.02, +4.80, and +2.36 points on CIFAR10, CIFAR100, and ImageNet100, respectively. This suggests that a fixed prior can be effective when its assumed structure is already well matched to the task, but becomes limiting when the appropriate interaction pattern differs from the predefined bias.

Overall, these results motivate making the structural prior learnable: rather than relying on the initial mask to be exactly correct, the model can preserve the useful inductive bias provided at initialization while adapting its interaction structure to the statistics of each dataset.

\begin{table}[!ht]
  \centering
  \caption{Best Top-1 accuracy (\%) on vision benchmarks using
  ViT-Tiny, reported as mean $\pm$ sample standard deviation
  across three random seeds. For each seed, we report the highest
  accuracy achieved during training. Values in parentheses denote
  improvement over Default in percentage points.
  $^\dagger$ denotes a single-seed result (seed 0).}
  \label{tab:ablation_cv}
  \resizebox{\linewidth}{!}{
  \begin{tabular}{lccccc}
  \toprule
  Dataset & Default & Mimetic & Impulse & Fixed Mask & Logical Mask \\
  \midrule
  CIFAR10
  & $93.93 \pm 0.11$
  & $94.98 \pm 0.26$ (+1.05)
  & $95.03 \pm 0.20$ (+1.10)
  & $94.60 \pm 0.42$ (+0.67)
  & $\mathbf{95.95 \pm 0.12}$ (+2.02) \\

  CIFAR100
  & $77.26 \pm 0.22$
  & $79.01 \pm 0.06$ (+1.75)
  & $79.62 \pm 0.35$ (+2.36)
  & $79.84 \pm 0.23$ (+2.58)
  & $\mathbf{82.06 \pm 0.14}$ (+4.80) \\

  ImageNet100
  & $83.47 \pm 0.13$
  & $84.55 \pm 0.14$ (+1.08)
  & $83.51 \pm 0.95$ (+0.04)
 & $83.87 \pm 0.25$ (+0.40)
  & $\mathbf{85.83 \pm 0.56}$ (+2.36) \\
  \bottomrule
  \end{tabular}
  }
  \end{table}

\subsection{Initialization Deviation}
\label{sec:initialization_deviation}

We investigate whether the performance difference between QK-Init and logical mask can be explained by their accuracy in reproducing the desired attention pattern at initialization. QK-Init fits the query and key projections to pattern-derived logits, whereas logical mask encodes the pattern through a
trainable additive attention bias.

To quantify initialization deviation, we use KL divergence from the desired attention distribution to the actual initialized attention. The desired distribution assigns equal probability to all allowed keys in each query row. We compute KL separately for each input, layer, head, and query position, then average the resulting values. Lower divergence indicates closer agreement with the desired pattern.

We evaluate both methods using four-layer, three-head transformers with sequence length 15, learnable positional embeddings, mean pooling, and no CLS token. We use seed 0 and initialization strength 20. Evaluation covers all 32,734 binary inputs that were not used to fit the QK initialization. Each layer is evaluated against its corresponding task pattern.

\begin{table}[!ht]
\centering
\caption{
Mean KL divergence between the desired attention distribution and initialized attention. Lower KL divergence indicates closer agreement with the desired distribution.
}
\label{tab:initialization_deviation}
\begin{tabular}{lcc}
\toprule
Task & QK-Init $\downarrow$ & Logical Mask $\downarrow$ \\
\midrule
2Parity        & $7.04 \times 10^{-4}$   & $7.08 \times 10^{-4}$ \\
ThreeSym       & $6.75 \times 10^{-4}$   & $6.84 \times 10^{-4}$ \\
3-Bit Majority & $7.05 \times 10^{-4}$   & $7.15 \times 10^{-4}$ \\
Cyclic3        & $1.07 \times 10^{-2}$ & $2.38 \times 10^{-2}$ \\
\bottomrule
\end{tabular}
\end{table}

Both methods closely reproduce the desired distributions for 2Parity, ThreeSym, and 3-Bit Majority, with similar initialization deviations. For Cyclic3, QK-Init achieves lower divergence than logical mask. These results suggest that any downstream advantage of logical mask is not explained by more accurate initial pattern matching in this setting.

\subsection{Effect of Initialization Strength}
\label{sec:init_strength}

We study the sensitivity of logical mask to the magnitude of the initialized prior. Specifically, for an initialization magnitude $c$, the open positions of the mask are initialized with $+c$, while the closed positions are initialized with $-c$. A larger magnitude therefore corresponds to a stronger structural prior at initialization. We evaluate different initialization strengths on both Boolean GOTU tasks and ImageNet-100.

\paragraph{Boolean tasks.}
We first evaluate initialization values of $\pm 1$, $\pm 5$, $\pm 10$, and $\pm 20$ on the Boolean tasks. The results are reported in Table~\ref{tab:mask_init_strength_boolean}.

\begin{table}[!ht]
  \centering
  \caption{
  Ablation on logical mask strength for Boolean GOTU tasks. Results are reported as
  mean test accuracy (\%) $\pm$ standard deviation. Higher is better.
  }
  \label{tab:mask_init_strength_boolean}
  \begin{tabular}{llcccc}
  \toprule
  Task & Split & $\pm 1$ & $\pm 5$ & $\pm 10$ & $\pm 20$ \\
  \midrule
  2Parity & Seen
  & $100.00 \pm 0.00$
  & $100.00 \pm 0.00$
  & $100.00 \pm 0.00$
  & $100.00 \pm 0.00$ \\
  & Unseen
  & $0.00 \pm 0.00$
  & $0.00 \pm 0.00$
  & $58.97 \pm 53.87$
  & $100.00 \pm 0.00$ \\
  \midrule
  Cyclic3 & Seen
  & $95.60 \pm 2.38$
  & $98.76 \pm 1.03$
  & $98.21 \pm 0.83$
  & $99.24 \pm 0.00$ \\
  & Unseen
  & $20.40 \pm 35.12$
  & $92.44 \pm 13.96$
  & $96.80 \pm 1.69$
  & $99.24 \pm 0.00$ \\
  \midrule
  3-Bit Majority & Seen
  & $100.00 \pm 0.00$
  & $100.00 \pm 0.00$
  & $100.00 \pm 0.00$
  & $100.00 \pm 0.00$ \\
  & Unseen
  & $69.95 \pm 27.43$
  & $89.93 \pm 22.51$
  & $100.00 \pm 0.00$
  & $100.00 \pm 0.00$ \\
  \bottomrule
  \end{tabular}
  \end{table}

These results suggest that the structural prior must be sufficiently strong at initialization to reliably influence the learned extrapolation behavior. The required strength is task dependent: some tasks, such as Cyclic3, benefit substantially from moderate initialization strengths, whereas 2Parity requires a stronger prior to consistently recover the target rule.

\paragraph{ImageNet-100.}
We further evaluate initialization strengths of $\pm 1$, $\pm 2$, $\pm 5$, $\pm 10$, and $\pm 20$ on ImageNet-100. The results are reported in Table~\ref{tab:mask_init_strength_imagenet100}.

\begin{table}[!ht]
\centering
\caption{
Ablation on logical mask strength for ImageNet-100. Results are reported as mean best Top-1 accuracy (\%) $\pm$ standard deviation. Higher is better.
}
\label{tab:mask_init_strength_imagenet100}
\begin{tabular}{lc}
\toprule
Initialization Values & Best Top-1 Accuracy \\
\midrule
$\pm 1$ & $84.10 \pm 0.45$ \\
$\pm 2$ & $84.23 \pm 0.15$ \\
$\pm 5$ & $85.01 \pm 0.35$ \\
$\pm 10$ & $85.83 \pm 0.56$ \\
$\pm 20$ & $85.41 \pm 0.27$ \\
\bottomrule
\end{tabular}
\end{table}

A similar trend is observed on ImageNet-100. Increasing the initialization magnitude from $\pm 1$ to $\pm 10$ improves the best Top-1 accuracy from $84.10\%$ to $85.83\%$. Further increasing the magnitude to $\pm 20$ does not provide an additional improvement, with the accuracy slightly decreasing to $85.41\%$. 

Overall, these experiments show that performance is sensitive to very weak initialization values, but becomes relatively stable once the structural prior is sufficiently strong.

\section{Proof of Theorem 1 and Corollary 1}
\label{app:proof2}
\begin{theorem}
\label{thm:masked_attention_jacobian}
Consider a single attention row with query $q\in\mathbb{R}^d$, keys $\{k_j\}_{j=1}^N$, and mask parameters $\{m_j\}_{j=1}^N$. Define the softmax function working on the attention score as
\[
p_j=
\frac{\exp(q^\top k_j/\sqrt d+\log\sigma(m_j))}
{\sum_{\ell=1}^N\exp(q^\top k_\ell/\sqrt d+\log\sigma(m_\ell))}.
\]
Then, for any $i,j\in\{1,\ldots,N\}$,
\begin{align}
\frac{\partial p_i}{\partial q}
&=
\frac{p_i}{\sqrt d}
\left(
k_i-\sum_{\ell=1}^N p_\ell k_\ell
\right),
\\
\frac{\partial p_i}{\partial k_j}
&=
\frac{1}{\sqrt d}
p_i(\delta_{ij}-p_j)q,
\\
\frac{\partial p_i}{\partial m_j}
&=
p_i(\delta_{ij}-p_j)\sigma(-m_j).
\end{align}
\end{theorem}

\begin{proof}
We consider a single attention row with
\[
p_j = \frac{\exp\!\left(s_j/\sqrt d + \log\sigma(m_j)\right)}{\sum^N_{\ell=1} \exp\!\left(s_\ell/\sqrt d + \log\sigma(m_\ell)\right)}.
\]
Let
\[
z_j = s_j/\sqrt d + \log\sigma(m_j).
\]
Then $p = \mathrm{softmax}(z)$, and the standard softmax derivative gives
\[
\frac{\partial p_i}{\partial z_j}
=
p_i(\delta_{ij} - p_j).
\]

By the chain rule, we have
\[
\frac{\partial z_j}{\partial q} = \frac{1}{\sqrt d} k_j,
\qquad
\frac{\partial z_j}{\partial k_j} = \frac{1}{\sqrt d} q,
\qquad
\frac{\partial z_j}{\partial m_j} = \sigma(-m_j).
\]

Combining these yields
\[
\frac{\partial p_i}{\partial q}
=
\sum^N_{j=1} \frac{\partial p_i}{\partial z_j} \frac{\partial z_j}{\partial q}
=
\frac{p_i}{\sqrt d}
\left(k_i - \sum^N_{\ell=1} p_\ell k_\ell\right),
\]
\[
\frac{\partial p_i}{\partial k_j}
=
\frac{1}{\sqrt d} p_i(\delta_{ij} - p_j) q,
\]
and
\[
\frac{\partial p_i}{\partial m_j}
=
p_i(\delta_{ij} - p_j)\sigma(-m_j),
\]
which completes the proof.
\end{proof}

\begin{corollary}
\label{cor:masked_attention_grad_norm}
Under the setting of Theorem $\ref{thm:masked_attention_jacobian}$, suppose
$\|q\|_2\le B_q$ and $\|k_j\|_2\le B_k$ for all $j$. Then
\begin{align}
\left\|\frac{\partial p_i}{\partial q}\right\|_2
&\le
\frac{2p_iB_k}{\sqrt d},
\\
\left\|\frac{\partial p_i}{\partial k_j}\right\|_2
&\le
\frac{p_i|\delta_{ij}-p_j|B_q}{\sqrt d}
\le
\frac{p_iB_q}{\sqrt d},
\\
\left|\frac{\partial p_i}{\partial m_j}\right|
&\le
p_i|\delta_{ij}-p_j|\sigma(-m_j)
\le
p_i .
\end{align}
\end{corollary}

\begin{proof}
From Theorem~\ref{thm:masked_attention_jacobian}, we have
\[
\frac{\partial p_i}{\partial q}
=
\frac{p_i}{\sqrt d}
\left(k_i - \sum^N_{\ell=1} p_\ell k_\ell\right).
\]
Assuming $\|k_j\|_2 \le B_k$ for all $j$, we bound by using triangular inequality and $\forall \ell\in[N], p_\ell\geq0,\sum^N_{\ell=1} p_\ell=1$:
\[
\left\|\sum^N_{\ell=1} p_\ell k_\ell\right\|_2
\le
\sum^N_{\ell=1} p_\ell \|k_\ell\|_2
\leq
\sum^N_{\ell=1} p_\ell B_k
= B_k,
\]
which implies (By triangular inequality)
\[
\left\|k_i - \sum^N_{\ell=1} p_\ell k_\ell\right\|_2
\le \|k_i\|_2 + \left\|\sum^N_{\ell=1} p_\ell k_\ell\right\|_2
\le 2B_k.
\]
Therefore,
\[
\left\|\frac{\partial p_i}{\partial q}\right\|_2
=
\left\|\frac{p_i}{\sqrt d}
\left(k_i - \sum^N_{\ell=1} p_\ell k_\ell\right)\right\|_2
=
\frac{|p_i|}{\sqrt d}\left\|\left(k_i - \sum^N_{\ell=1} p_\ell k_\ell\right)\right\|_2
\le
\frac{2p_i B_k}{\sqrt d}.
\]

Similarly, using $\|q\|_2 \le B_q$, we obtain
\[
\left\|\frac{\partial p_i}{\partial k_j}\right\|_2
=
\frac{1}{\sqrt d} p_i |\delta_{ij}-p_j| \|q\|_2
\le
\frac{p_i B_q}{\sqrt d}.
\]

Finally, since $\sigma(-m_j)\le 1$ and $|\delta_{ij}-p_j|\le 1$, we have
\[
\left|\frac{\partial p_i}{\partial m_j}\right|
\le p_i,
\]
which completes the proof.
\end{proof}

\section{Variable-Length Inputs and Mask Parameterizations}

A dense attention mask is tied to a fixed input length because it stores a separate parameter for each pair of tokens. To support variable-length inputs and reduce the number of mask parameters, we represent the mask using an implicit neural representation (INR).

\paragraph{INR mask parameterization.}
An INR uses a neural network to map coordinates to values. Instead of learning each mask entry \(M_{ij}\) independently, we compute it as
\[
    M_{ij} = f_{\theta}(G_{ij}),
\]
where \(G_{ij}\) contains geometric information about tokens \(i\) and \(j\), and \(f_{\theta}\) is an MLP. Since the same MLP can be evaluated at any pair of coordinates, it can generate masks for different input lengths without storing a separate mask for each length. INRs are commonly used to represent continuous signals, such as images, at different resolutions \citep{tancik2020fourier,sitzmann2020implicit}.

\paragraph{Variable-length arithmetic tasks.}
For the arithmetic experiments, \(G_{ij}\) contains the coordinates of the two tokens and their relative distance. We use a ReLU MLP with a hidden dimension of 64. We first train the MLP to approximate the target mask and then jointly optimize the MLP and the Transformer.

We train on inputs with different digit lengths on multiplication task in the same run. Table~\ref{tab:variable_digit_lengths} shows the results. The INR-based logical mask model achieves high digit-level and sequence-level accuracy across all tested input lengths, while the standard Transformer performs substantially worse.

\begin{table}[!ht]
\centering
\caption{
Performance on variable-length arithmetic tasks. \(d\) denotes the number of digits in each operand. Results are reported as digit-level and sequence-level accuracy.
}
\label{tab:variable_digit_lengths}
\begin{tabular}{clcc}
\toprule
Length & Method & Digit Acc. (\%) & Sequence Acc. (\%) \\
\midrule
\(d=2\) & MaskInit (INR) & \textbf{97.53} & \textbf{93.26} \\
\(d=2\) & Transformer    & 47.02          & 11.52          \\
\midrule
\(d=3\) & MaskInit (INR) & \textbf{99.99} & \textbf{99.96} \\
\(d=3\) & Transformer    & 57.08          & 3.43           \\
\midrule
\(d=4\) & MaskInit (INR) & \textbf{99.92} & \textbf{99.60} \\
\(d=4\) & Transformer    & 43.95          & 0.12           \\
\bottomrule
\end{tabular}
\end{table}

\paragraph{Comparison with relative position bias.}
We further compare the INR mask with relative position bias (RPB), as used in Swin Transformer \citep{liu2021swin}. RPB uses a learnable lookup table indexed by the relative spatial offset between two tokens. In contrast, our INR predicts the mask value from geometric features.

For each query-key pair, the INR receives the normalized token coordinates, relative spatial offsets, relative distance, and binary indicators for interactions involving the CLS token. A lightweight MLP hen produces one mask value for each attention head. As before, we first fit the INR to the target logical mask and then jointly train it with the Transformer.

On ImageNet-100, MaskInit (INR) achieves a top-1 accuracy of \(87.32 \pm 0.05\%\), outperforming both RPB variants with a similar number of parameters. Results are summarized in Table~\ref{tab:rpb_inr}.

\begin{table}[!ht]
\centering
\caption{ Comparison of mask parameterizations on ImageNet-100. Accuracy is reported as the mean and sample standard deviation of the best validation accuracy across three random seeds.}
\label{tab:rpb_inr}
\begin{tabular}{lcc}
\toprule
Method & Top-1 Accuracy (\%) \(\uparrow\) & Parameters (M) \\
\midrule
RPB (Default)      & \(86.23 \pm 0.13\)          & 5.57 \\
RPB (Impulse Init) & \(86.53 \pm 0.06\)          & 5.57 \\
INR MaskInit (Impulse Init)     & \(\mathbf{87.32 \pm 0.05}\) & 5.60 \\
\bottomrule
\end{tabular}
\end{table}

\section{Empirical Verification of Persistency}
\label{sec:empirical_persistency}

We empirically examine the persistency of the initialized structural
patterns during training on four Boolean tasks. We compare logical mask
with Structured (SVD) and Structured (Opt) over five random seeds
(0-4), measuring the KL divergence from the initial attention pattern
to the current pattern at epochs 10, 30, 50, 70, and the final epoch
(index 89). Smaller KL divergence indicates greater preservation of
the initialized pattern. We additionally track the gradient norm of
the mask parameters for logical mask and the query-key parameters for
the structured initialization baselines.

\paragraph{Persistence of initialized patterns.}
As shown in Table~\ref{tab:persistency_kl}, Logical mask remains close to
its initialization across all sampled epochs. Its mean KL divergence
increases gradually, reaching only 0.052, 0.032, 0.208, and 0.011 at
the final epoch for 2Parity, Cyclic3, 3-Bit Majority, and ThreeSym,
respectively. In contrast, both structured initialization baselines
already exhibit substantial divergence by epoch 10 and retain
considerably larger divergence throughout the sampled trajectory.
Although their divergence sometimes decreases later in training,
their final values remain between 2.24 and 2.87 across the four tasks.
These observations support greater empirical persistency of the
patterns encoded by logical mask.

\begin{table}[!ht]
\centering
\caption{KL divergence from the initial to the current attention
pattern at selected training epochs. Values are means over five
random seeds. Final denotes epoch index 89.}
\label{tab:persistency_kl}
\setlength{\tabcolsep}{5pt}
\begin{tabular}{llccccc}
\toprule
Task & Method & 10 & 30 & 50 & 70 & Final \\
\midrule
2Parity
& Structured (SVD) & 2.88 & 2.74 & 2.74 & 2.77 & 2.76 \\
& Structured (Opt) & 2.83 & 2.67 & 2.69 & 2.72 & 2.69 \\
& Logical Mask & \textbf{0.00} & \textbf{0.01} &
\textbf{0.02} & \textbf{0.04} & \textbf{0.05} \\
\midrule
Cyclic3
& Structured (SVD) & 1.36 & 2.16 & 2.39 & 2.29 & 2.24 \\
& Structured (Opt) & 4.46 & 3.70 & 3.15 & 2.81 & 2.66 \\
& Logical Mask & \textbf{0.00} & \textbf{0.00} &
\textbf{0.01} & \textbf{0.02} & \textbf{0.03} \\
\midrule
3-Bit Majority
& Structured (SVD) & 2.99 & 2.75 & 2.91 & 2.91 & 2.87 \\
& Structured (Opt) & 2.92 & 2.69 & 2.68 & 2.69 & 2.69 \\
& Logical Mask & \textbf{0.00} & \textbf{0.02} &
\textbf{0.07} & \textbf{0.16} & \textbf{0.21} \\
\midrule
ThreeSym
& Structured (SVD) & 2.72 & 2.77 & 2.73 & 2.73 & 2.73 \\
& Structured (Opt) & 2.83 & 2.73 & 2.68 & 2.67 & 2.67 \\
& Logical Mask & \textbf{0.00} & \textbf{0.00} &
\textbf{0.01} & \textbf{0.01} & \textbf{0.01} \\
\bottomrule
\end{tabular}
\end{table}

\paragraph{Gradient dynamics.}
Table~\ref{tab:persistency_grad} reports the corresponding gradient
norms. At every sampled epoch, the mask-parameter gradient norm is
several orders of magnitude smaller than the query-key gradient
norms of the structured initialization baselines. This separation
coincides with the smaller attention-pattern drift observed for
logical mask. The mask gradients remain nonzero, indicating that the
initialized structure remains trainable rather than being frozen.
However, these norms concern different parameterizations and
parameter counts, so their magnitudes are descriptive diagnostics,
not directly comparable measures of functional change. This also provides empirical evidence for our propositions.

\begin{table}[!t]
\centering
\caption{Gradient norms at selected training epochs, averaged over
training batches within each epoch and then over five random seeds.
Logical mask reports mask-parameter gradient norms; Structured (SVD) and Structured (Opt) report query-key gradient norms.}
\label{tab:persistency_grad}
\setlength{\tabcolsep}{5pt}

\resizebox{\linewidth}{!}{%
\begin{tabular}{llccccc}
\toprule
Task & Method & 10 & 30 & 50 & 70 & Final \\
\midrule
2Parity
& Structured (SVD) & 0.09 & 0.05 & 0.01 & 0.01 & 0.01 \\
& Structured (Opt) & 0.09 & 0.06 & 0.06 & 0.06 & 0.05 \\
& Logical Mask
& $8.17 \times 10^{-6}$ & $5.48 \times 10^{-6}$
& $5.40 \times 10^{-6}$ & $8.51 \times 10^{-6}$
& $1.18 \times 10^{-5}$ \\
\midrule
Cyclic3
& Structured (SVD) & 3.19 & 17.28 & 10.25 & 9.30 & 6.61 \\
& Structured (Opt) & 6.96 & 8.01 & 8.80 & 7.34 & 6.65 \\
& Logical Mask
& $2.01 \times 10^{-3}$ & $1.78 \times 10^{-3}$
& $3.12 \times 10^{-3}$ & $5.71 \times 10^{-3}$
& $8.17 \times 10^{-3}$ \\
\midrule
3-Bit Majority
& Structured (SVD) & 0.22 & 0.03 & 0.04 & 0.04 & 0.03 \\
& Structured (Opt) & 0.25 & 0.02 & 0.03 & 0.03 & 0.02 \\
& Logical Mask
& $8.18 \times 10^{-6}$ & $9.91 \times 10^{-5}$
& $6.11 \times 10^{-5}$ & $8.44 \times 10^{-5}$
& $1.20 \times 10^{-4}$ \\
\midrule
ThreeSym
& Structured (SVD) & 0.24 & 0.13 & 0.13 & 0.08 & 0.09 \\
& Structured (Opt) & 0.34 & 0.35 & 0.22 & 0.33 & 0.35 \\
& Logical Mask
& $1.20 \times 10^{-5}$ & $3.39 \times 10^{-5}$
& $5.27 \times 10^{-5}$ & $5.45 \times 10^{-5}$
& $5.29 \times 10^{-5}$ \\
\bottomrule
\end{tabular}%
}
\end{table}

\section{Design of Mask Pattern}
\label{app:init_mask}
In this section, we outline the key principles for designing effective mask-based initialization. Before presenting the design principles, we first note that low-order interactions, such as pairwise terms $x_i x_j$, can be easily represented by shallow neural networks. 
In particular, it is well known that a two-layer MLP can express functions such as XOR, indicating that modeling such interactions does not require deep architectural complexity.

This observation suggests that the role of attention can choose not to represent these polynomial interactions directly, but rather to select which variables should interact. 
However, prior work \citep{abbe2024generalization} failure indicates that transformers may rely on spurious correlations that fit the training data, rather than recovering the underlying structure.

Therefore, by explicitly controlling token-wise interactions through the attention mask, we can guide the model to construct the desired extrapolation property in principle.

\paragraph{Design Principles for Mask Pattern.}
Mask-based initialization is guided by the following principles.

\textbf{(1) Task-aligned interactions.}
The mask should allow only the interactions needed for the task.
For example, in 2-parity, only the first two input components should
interact. Other components may affect the prediction and hurt extrapolation.

\textbf{(2) Suppression of shortcut solutions.}
Stacking multiple layers allows the model to build higher-order
interactions from simpler ones. We can guide this process by blocking
unnecessary interactions. For 3-Bit Majority, a second-order interaction is an MDI solution. Blocking off-diagonal interactions therefore encourages the model to learn the third order solution and extrapolate correctly.

\textbf{(3) Multi-head specialization.}
A layer contains multiple attention heads, and each head can capture a
different interaction. For example, ThreeSym is composed of multiple
2-parity functions. We therefore assign different masks to different
heads to support correct extrapolation.

These principles define a structured way to control the interaction graph of the model, guiding it toward the desired computation.



\begin{figure}[!ht]
\centering
\includegraphics[width=0.85\linewidth]{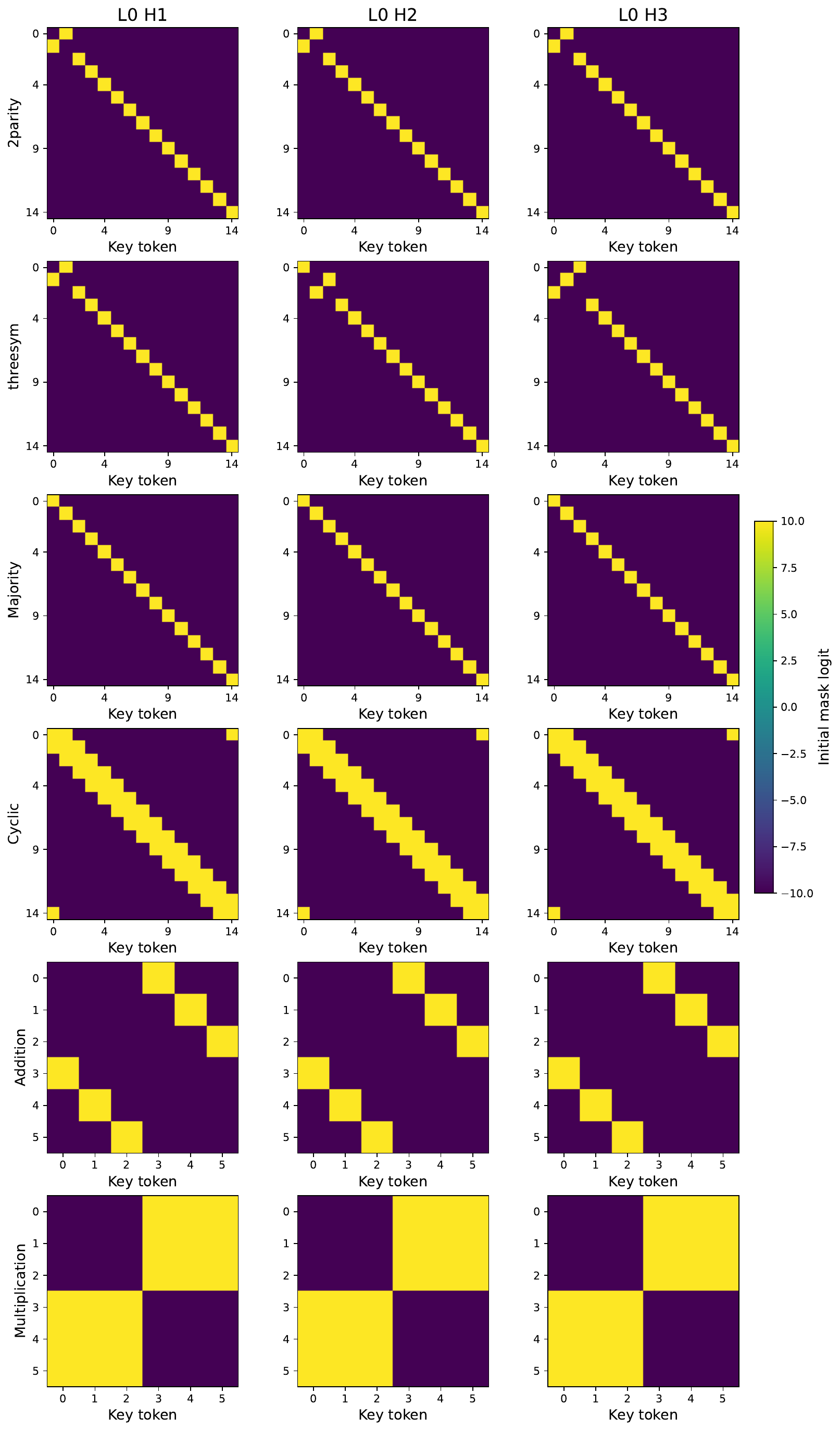}
\caption{
Initial mask logits for different tasks. Each row corresponds to a task and each column shows different attention heads at the first layer.
}
\label{fig:initial_mask}
\end{figure}
\clearpage

\section{Experiment Details}
\label{app:exp_de}
In this section, we provide a detailed description of the experimental setup. Boolean experiments are conducted on a single NVIDIA 4090 GPU. Other experiments are conducted on a single NVIDIA A100 GPU (40GB).

\subsection{Boolean Problems}
For Boolean tasks, we build our implementation upon the GOTU repository\footnote{\url{https://github.com/aryol/GOTU/tree/main}}, with additional modifications such as incorporating a learning rate scheduler to improve convergence speed. The Structured (Opt) initialization is based on the official implementation\footnote{\url{https://github.com/osiriszjq/structured_initialization}}, which optimizes the $QK$ matrices directly via gradient descent, instead of using SVD-based decomposition. The detailed information is shown at \Cref{tab:boolean_setup}. For logical mask, we choose open positions = 20 and close positions = -20.

\begin{table}[!h]
\centering
\caption{Training setup for Boolean tasks.}
\label{tab:boolean_setup}
\begin{tabular}{lccccccc}
\toprule
Task & Epochs & Layers & Heads & LR & Batch & Optimizer & Seed \\
\midrule
2Parity      & 90  & 4 & 3 & $1\times10^{-3}$ --$1\times10^{-4}$ & 256 & Adam & 0--4 \\
ThreeSym     & 90  & 4 & 3 & $1\times10^{-3}$ --$1\times10^{-4}$ & 256 & Adam& 0--4 \\
3-Bit Majority     & 90   & 4 & 3 & $1\times10^{-3}$ --$1\times10^{-4}$ & 256 & Adam& 0--4 \\
Cyclic3       & 90   & 4 & 3 & $1\times10^{-3}$ --$1\times10^{-5}$ & 256 & Adam& 0--4 \\
\bottomrule
\end{tabular}
\end{table}

\subsection{Arithmetic Problems}
We implement all arithmetic experiments from scratch due to the lack of existing code. All methods share the same architecture and training setup, differing only in the initialization strategy \Cref{tab:arith_setup}.

\begin{table}[!h]
\centering
\caption{Arithmetic task setup.}
\label{tab:arith_setup}
\begin{tabular}{ll}
\toprule
Item & Value \\
\midrule
Task & Addition / Multiplication (3-digit) \\
Dataset size & $10^6$ (all combinations) \\
Model & Transformer \\
Layers / Heads & 6 / 4 \\
Model dim & 256 \\
MLP dim & 256 \\
Head dim & 64 \\
Positional encoding & learnable \\
Loss & Cross-entropy \\
Optimizer & AdamW \\
Learning rate & $1\times10^{-3}$ \\
Weight decay & $1\times10^{-3}$ \\
Batch size & 256 \\
Epochs & 50 \\
Scheduler & cosine \\
\midrule
Logical Mask & open = 5, close = -5 \\
Structured init (SVD)& open = 10000, close = 0 \\
Structured init (Opt)& steps = 2000, lr = $1\times10^{-4}$\\
\bottomrule
\end{tabular}
\end{table}

\subsection{Image Classification}
We summarize the experimental setup for image classification tasks in \Cref{tab:init_1,tab:init_2,tab:init_3} and the corresponding dataset-specific tables. All experiments are conducted using a ViT-Tiny architecture with identical training pipelines across different initialization methods to ensure fair comparison.

For CIFAR10 and CIFAR100, we follow a standard training protocol with a batch size of 512, cosine learning rate schedule, and extensive data augmentation including Mixup, CutMix, and RandAugment. For ImageNet100, we adopt a similar setup with linear learning rate scaling and a shorter warmup period.

To evaluate the effect of initialization, we compare default initialization with prior structured initialization methods (mimetic and impulse) and our proposed logical mask. For logical mask, we use a simple prior with open entries set to $10$ and closed entries set to $-10$, which encodes a strong structural bias at initialization.

We evaluate our method on CIFAR10 \citep{Krizhevsky2009LearningML}, CIFAR100 \citep{Krizhevsky2009LearningML}, and ImageNet100 is constructed as a 100-class subset of ImageNet \citep{5206848}, following the version available at \url{https://huggingface.co/datasets/clane9/imagenet-100}. Our implementation is based on the official repository\footnote{\url{https://github.com/osiriszjq/structured_initialization}}.

\begin{table}[!h]
\centering
\caption{Experimental setup for CIFAR10.}
\label{tab:init_1}
\begin{tabular}{ll}
\toprule
Item & Value \\
\midrule
Model & vit\_tiny\_patch16\_224 \\
Dataset & CIFAR10 \\
Classes & 10 \\
Input size & $3\times224\times224$ \\
Seed & 0--2 \\
GPU & 1×A100 \\
Batch size & 512 \\
Epochs & 400 \\
Optimizer & AdamW \\
LR & $2\times10^{-3}$ \\
Scheduler & cosine \\
Warmup epochs & 50 \\
Weight decay & 0.05 \\
AMP & yes \\
Drop path & 0.1 \\
Label smoothing & 0.1 \\
Mixup / CutMix & 0.8 / 1.0 \\
Augment & RandAug (m9) \\
\bottomrule
\end{tabular}
\end{table}

\begin{table}[!h]
\centering
\label{tab:init_2}
\caption{Experimental setup for CIFAR100.}
\begin{tabular}{ll}
\toprule
Item & Value \\
\midrule
Model & vit\_tiny\_patch16\_224 \\
Dataset & CIFAR100 \\
Classes & 100 \\
Input size & $3\times224\times224$ \\
Seed & 0--2 \\
GPU & 1×A100 \\
Batch size & 512 \\
Epochs & 400 \\
Optimizer & AdamW \\
LR & $2\times10^{-3}$ \\
Scheduler & cosine \\
Warmup epochs & 50 \\
Weight decay & 0.05 \\
AMP & yes \\
Drop path & 0.1 \\
Label smoothing & 0.1 \\
Mixup / CutMix & 0.8 / 1.0 \\
Augment & RandAug (m9) \\
\bottomrule
\end{tabular}
\end{table}

\begin{table}[!h]
\centering
\label{tab:init_3}
\caption{Experimental setup for ImageNet100.}
\begin{tabular}{ll}
\toprule
Item & Value \\
\midrule
Model & vit\_tiny\_patch16\_224 \\
Dataset & ImageNet100 \\
Classes & 100 \\
Input size & $3\times224\times224$ \\
Seed & 0--2 \\
GPU & 1×A100 \\
Batch size & 512 \\
Epochs & 300 \\
Optimizer & AdamW \\
Base LR & $1\times10^{-3}$ \\
LR scaling & linear (base size 512) \\
Scheduler & cosine \\
Warmup epochs & 5 \\
Cooldown epochs & 10 \\
Weight decay & 0.05 \\
AMP & yes \\
Global pool & avg \\
Drop path & 0.1 \\
Label smoothing & 0.1 \\
Mixup / CutMix & 0.8 / 1.0 \\
Augment & RandAug (m9) \\
\bottomrule
\end{tabular}
\end{table}
\newpage

\section{Visualizations}
\subsection{Visualization of Structure Initialization}
\label{app:impulse}
We visualize structured initialization in \Cref{fig:structured_init}, where convolutional impulse filters are translated into corresponding attention masks.
\begin{figure}[!ht]
    \centering
    \includegraphics[width=0.9\linewidth]{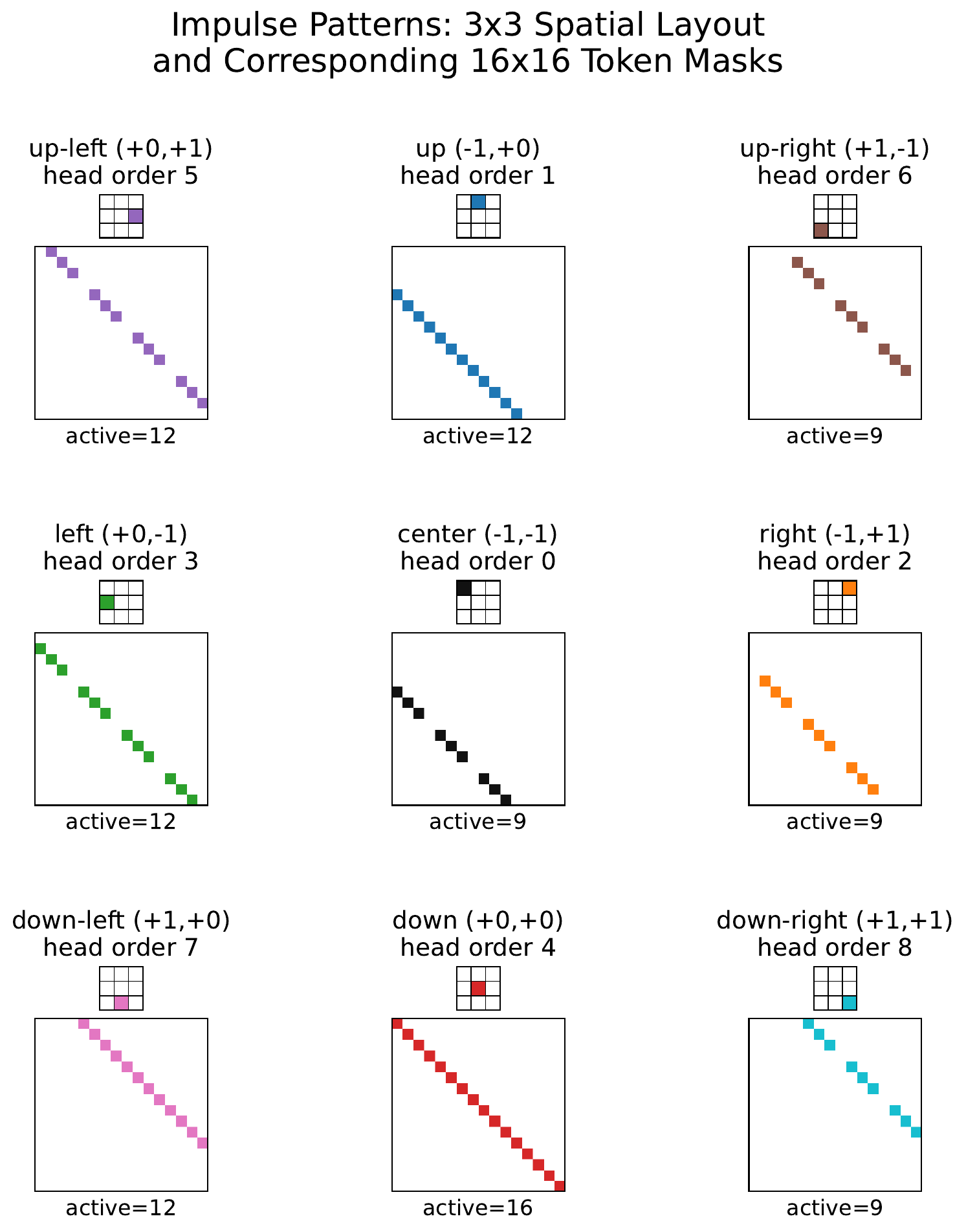}
    \caption{
    Structured initialization for vision transformers. 
    convolutional impulse filters (above) are translated into corresponding token-level attention masks (below), 
    inducing locality and spatial inductive bias.
    }
    \label{fig:structured_init}
\end{figure}
\subsection{Mask For Image Classification}
\label{app:trained_mask}
We visualize the learned masks for Head 1 at different layers after training for different tasks, as shown in \Cref{fig:mask_visualization}. 
The plots illustrate the effective bias applied to the attention logits. 
We observe that when the initial inductive bias is not well aligned with the task, the learnable mask adapts and modifies the pattern accordingly.
\begin{figure}[!ht]
\centering
\includegraphics[width=0.9\linewidth]{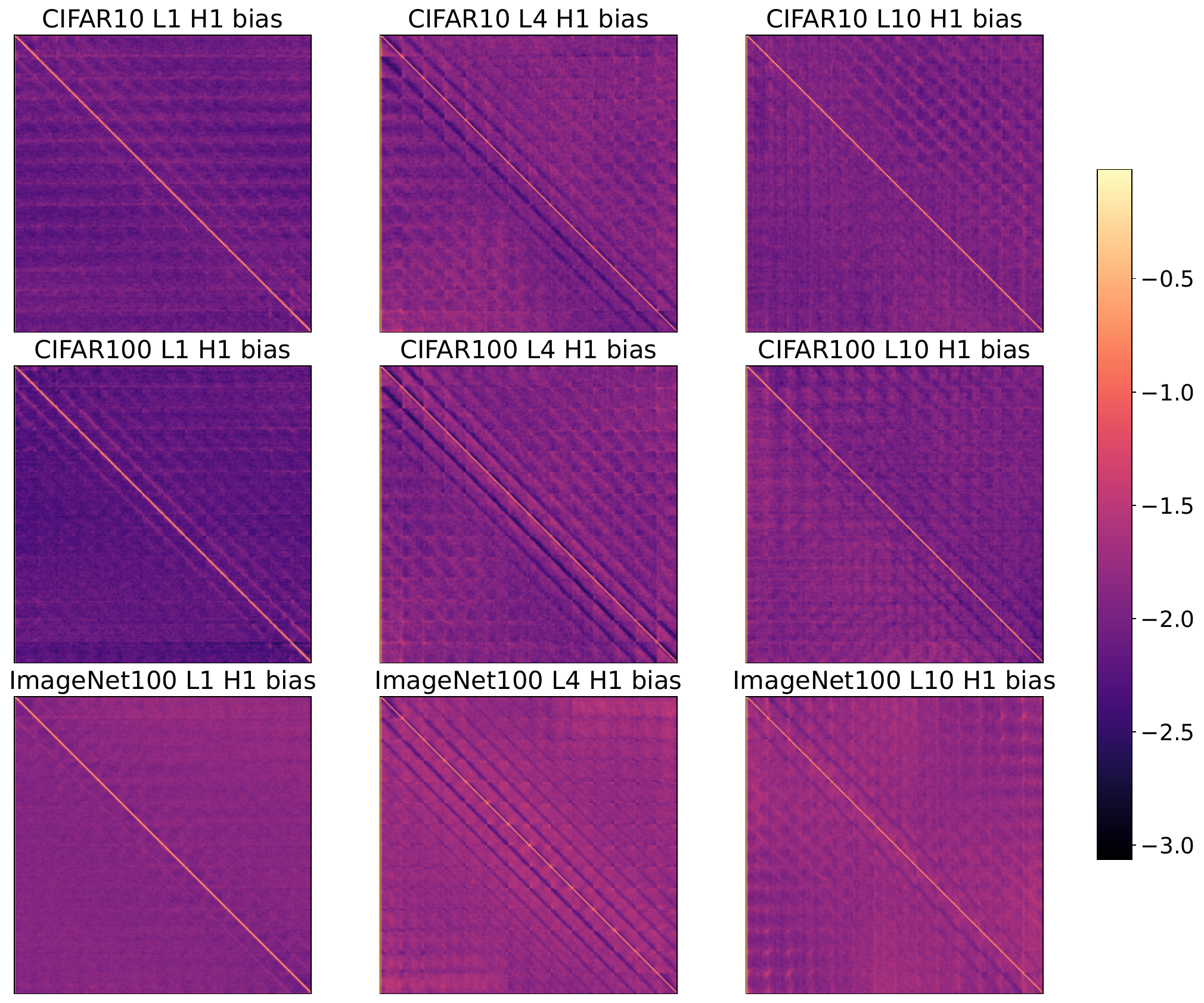}
\caption{
Learned mask (after training) for CIFAR10, CIFAR100, and ImageNet100 across different layers. The visualization shows the effective mask bias applied to the attention logits.
}
\label{fig:mask_visualization}
\end{figure}

\subsection{Attention Map For Image Classification}
\label{app:trained_img}
We visualize the attention maps for different methods in \Cref{fig:attn_cifar10,fig:attn_cifar100,fig:attn_imagenet100}. All attention values are taken after the softmax operation for the first test sample. For mask-based initialization, the mask is applied before softmax, and the resulting attention probabilities are shown. We observe that the core pattern is preserved across all tasks, while additional positions are gradually opened depending on the alignment between the task, layer, and the prior.

\begin{figure}[!ht]
\centering
\includegraphics[width=1.0\linewidth]{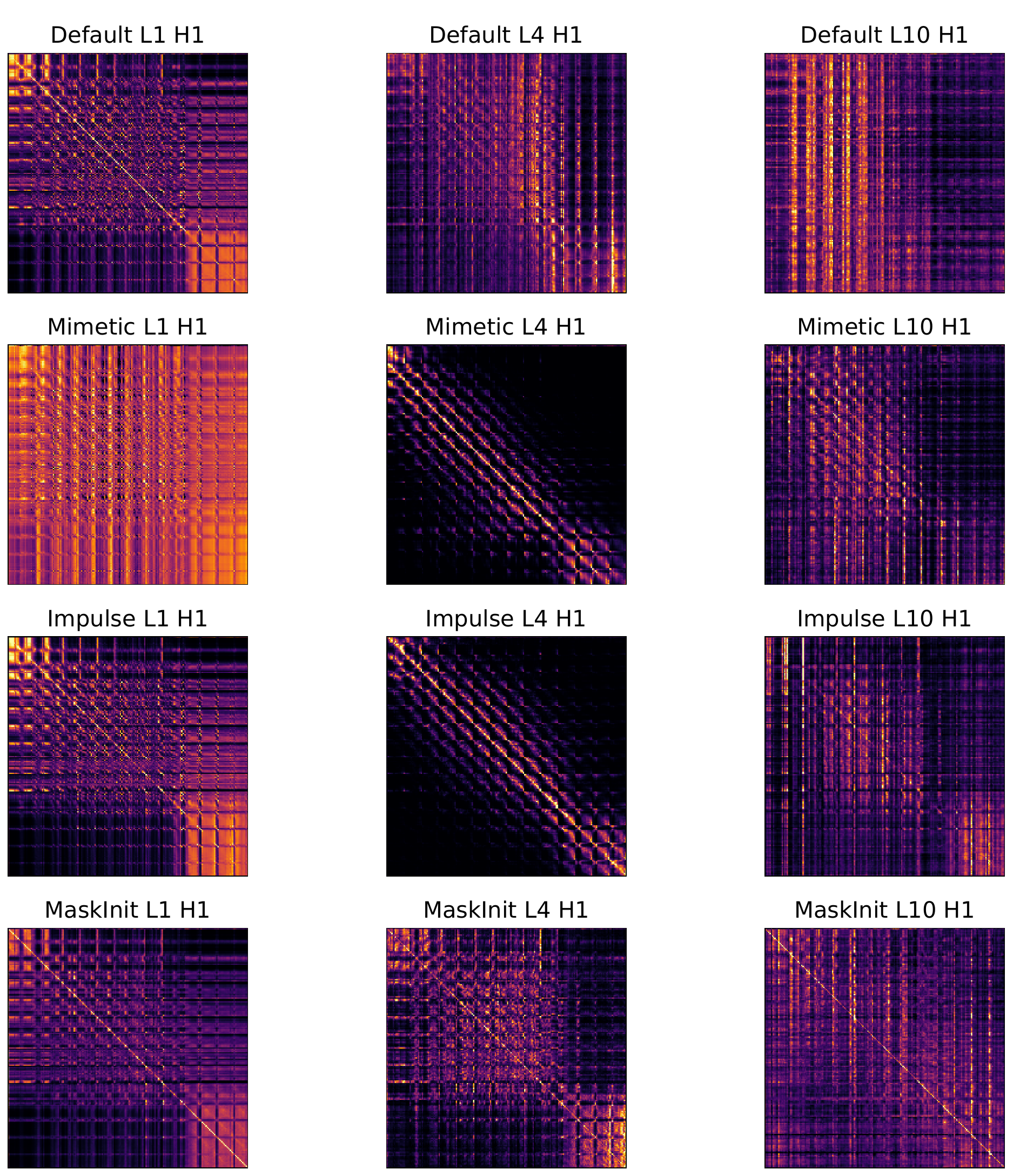}
\caption{
Attention maps for the first validation sample on CIFAR10.
}
\label{fig:attn_cifar10}
\end{figure}

\newpage
\begin{figure}[!ht]
\centering
\includegraphics[width=1.0\linewidth]{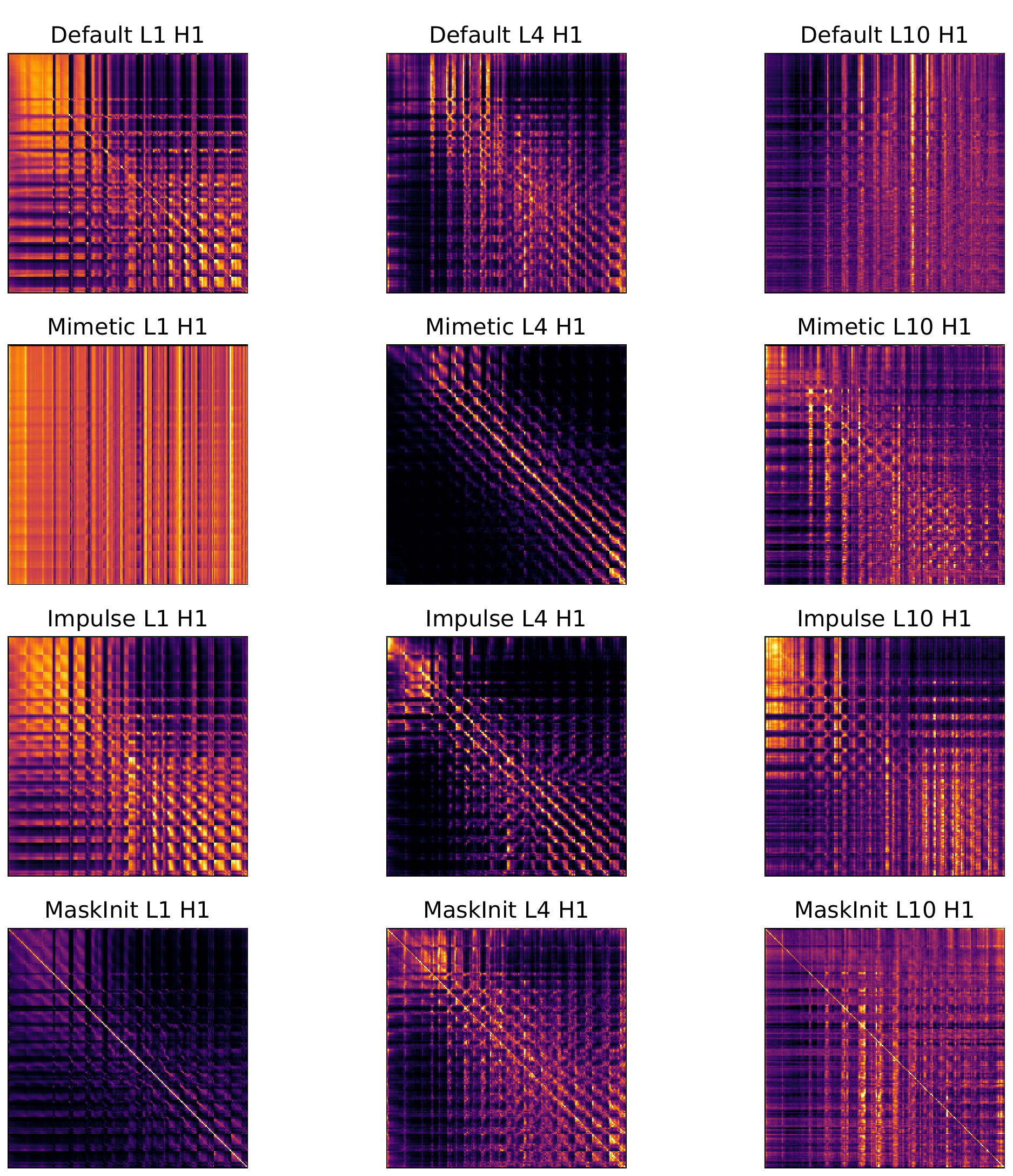}
\caption{
Attention maps for the first validation sample on CIFAR100.
}
\label{fig:attn_cifar100}
\end{figure}
\clearpage
\begin{figure}[!ht]
\centering
\includegraphics[width=1.0\linewidth]{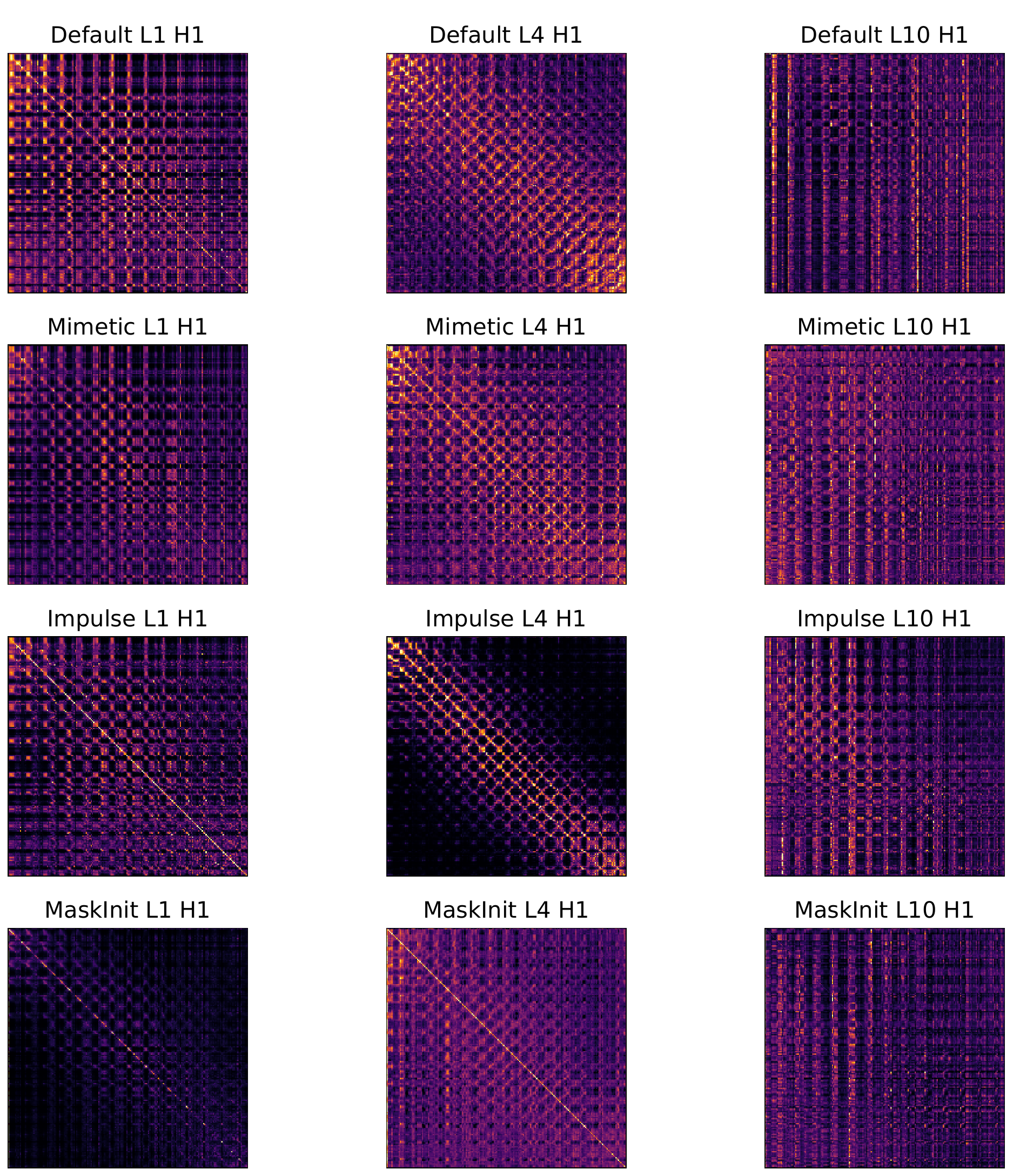}
\caption{
Attention maps for the first validation sample on ImageNet100.
}
\label{fig:attn_imagenet100}
\end{figure}
\clearpage

\clearpage

\subsection{Visualizations for Convergence Curves}
In this section, we visualize the convergence behavior of different initialization methods on the CIFAR100 and ImageNet100 datasets as shown at \Cref{fig:cv_results}. 
Our method begins to outperform the baselines after around 100 epochs and maintains a faster convergence rate throughout training. 
This suggests a two-phase training dynamic: in the early stage, the model aligns latent features under the constraint of the mask, while in later stages, it jointly refines both the mask and the latent representations.
\begin{figure}[!ht]
\centering
\includegraphics[width=0.9\linewidth]{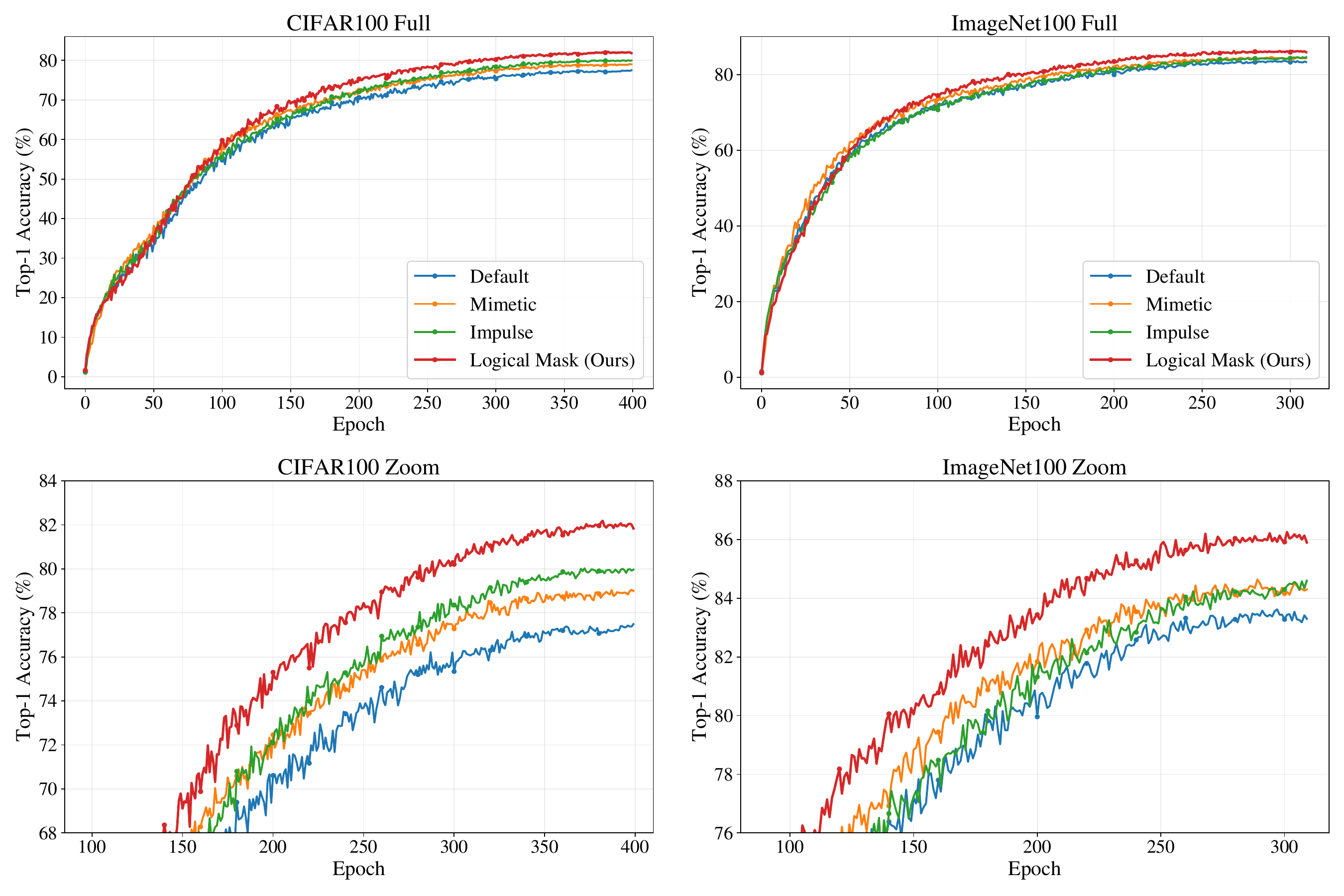}
\caption{
Training curves on CIFAR100 (left) and ImageNet100 (right) under identical training settings. The plots show validation Top-1 accuracy over training epochs. We compare default initialization, prior structured initialization methods, and the proposed logical mask. Our logical mask consistently converges faster in the early stages and achieves higher final accuracy. This indicates that mask-based initialization provides a better inductive bias, leading to improved optimization efficiency and generalization on vision benchmarks.
}
\label{fig:cv_results}
\end{figure}
\clearpage

\end{document}